\documentclass[conference]{IEEEtran}
\IEEEoverridecommandlockouts
\usepackage{cite}
\usepackage{amsmath,amssymb,amsfonts}
\usepackage{graphicx}
\usepackage{textcomp}
\usepackage{xcolor}
\usepackage{url} 
\usepackage{algorithm}
\usepackage{algpseudocode}
\usepackage{subcaption}
\usepackage{amsthm}
\usepackage{amssymb}
\usepackage{amsfonts}
\usepackage{amsmath}
\usepackage{multirow}
\usepackage{hyperref}

\newtheorem{theorem}{Theorem}
\def\BibTeX{{\rm B\kern-.05em{\sc i\kern-.025em b}\kern-.08em
    T\kern-.1667em\lower.7ex\hbox{E}\kern-.125emX}}

\begin{document}

\title{Angle based dynamic learning rate for gradient descent\\
\thanks{Acceted in IJCNN 2023. Codes:\href{https://github.com/misterpawan/dycent}{https://github.com/misterpawan/dycent}}
}

\author{\IEEEauthorblockN{Neel Mishra}
\IEEEauthorblockA{\textit{CSTAR} \\
\textit{IIIT Hyderabad}\\
Hyderabad, India \\
neel.mishra@research.iiit.ac.in}
\and
\and 
\IEEEauthorblockN{Pawan Kumar}
\IEEEauthorblockA{\textit{CSTAR} \\
\textit{IIIT Hyderabad}\\
Hyderabad, India \\
pawan.kumar@iiit.ac.in}
}


\maketitle

\begin{abstract}
In our work, we propose a novel yet simple approach to obtain an adaptive learning rate for gradient-based descent methods on classification tasks. Instead of the traditional approach of selecting adaptive learning rates via the decayed expectation of gradient-based terms, we use the angle between the current gradient and the new gradient: this new gradient is computed from the direction orthogonal to the current gradient, which further helps us in determining a better adaptive learning rate based on angle history, thereby, leading to relatively better accuracy compared to the existing state-of-the-art optimizers. 
On a wide variety of benchmark datasets with prominent image classification architectures such as ResNet, DenseNet, EfficientNet, and VGG, we find that our method leads to the highest accuracy in most of the datasets. Moreover, we prove that our method is convergent.
\end{abstract}

\begin{IEEEkeywords}
Adam, Image Classification, Optimization, Angle, Gradient Descent, Neural Networks, ResNet
\end{IEEEkeywords}

\begin{figure*}[ht]
  \centering
  \includegraphics[width=.24\linewidth]{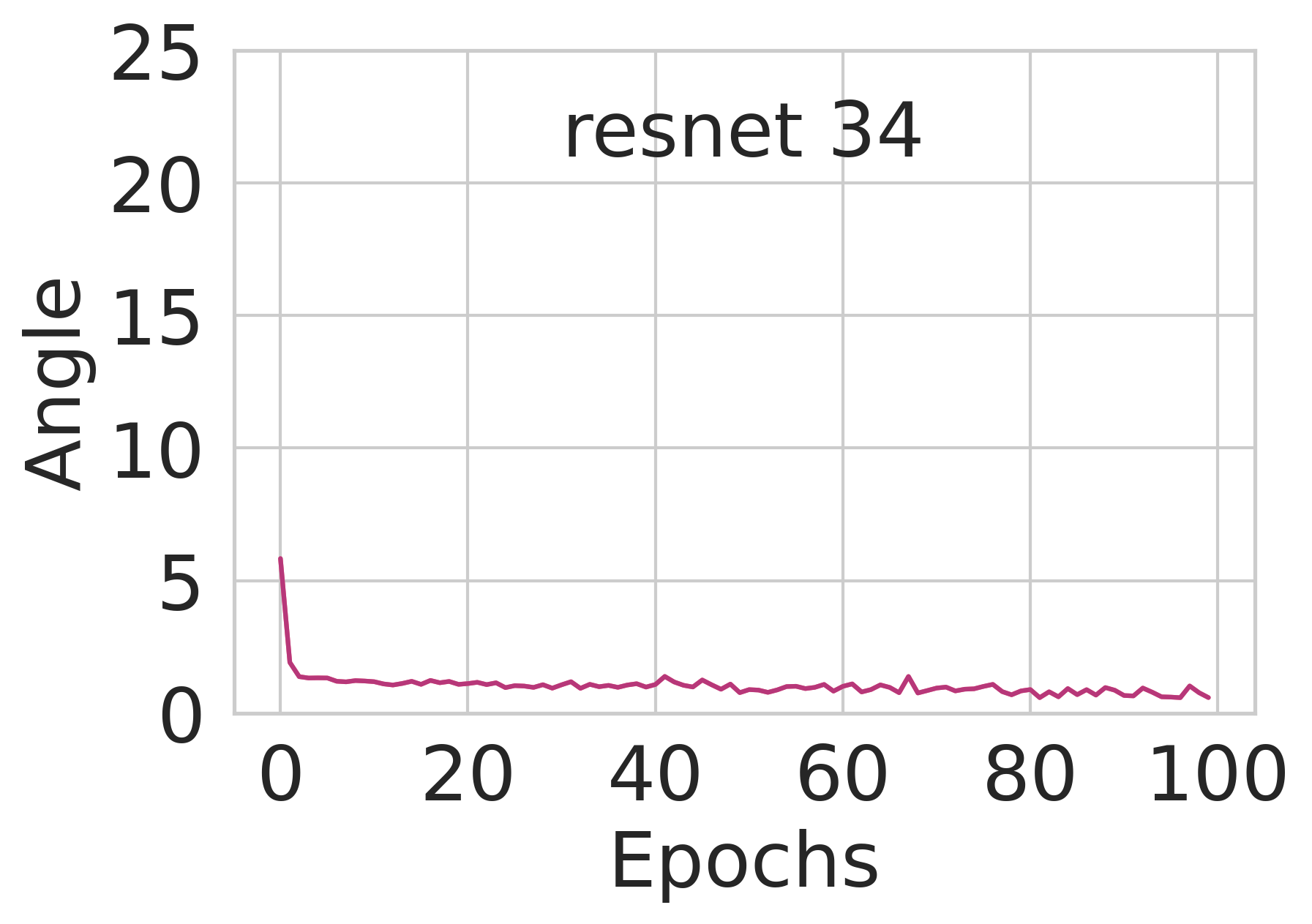}
  \includegraphics[width=.24\linewidth]{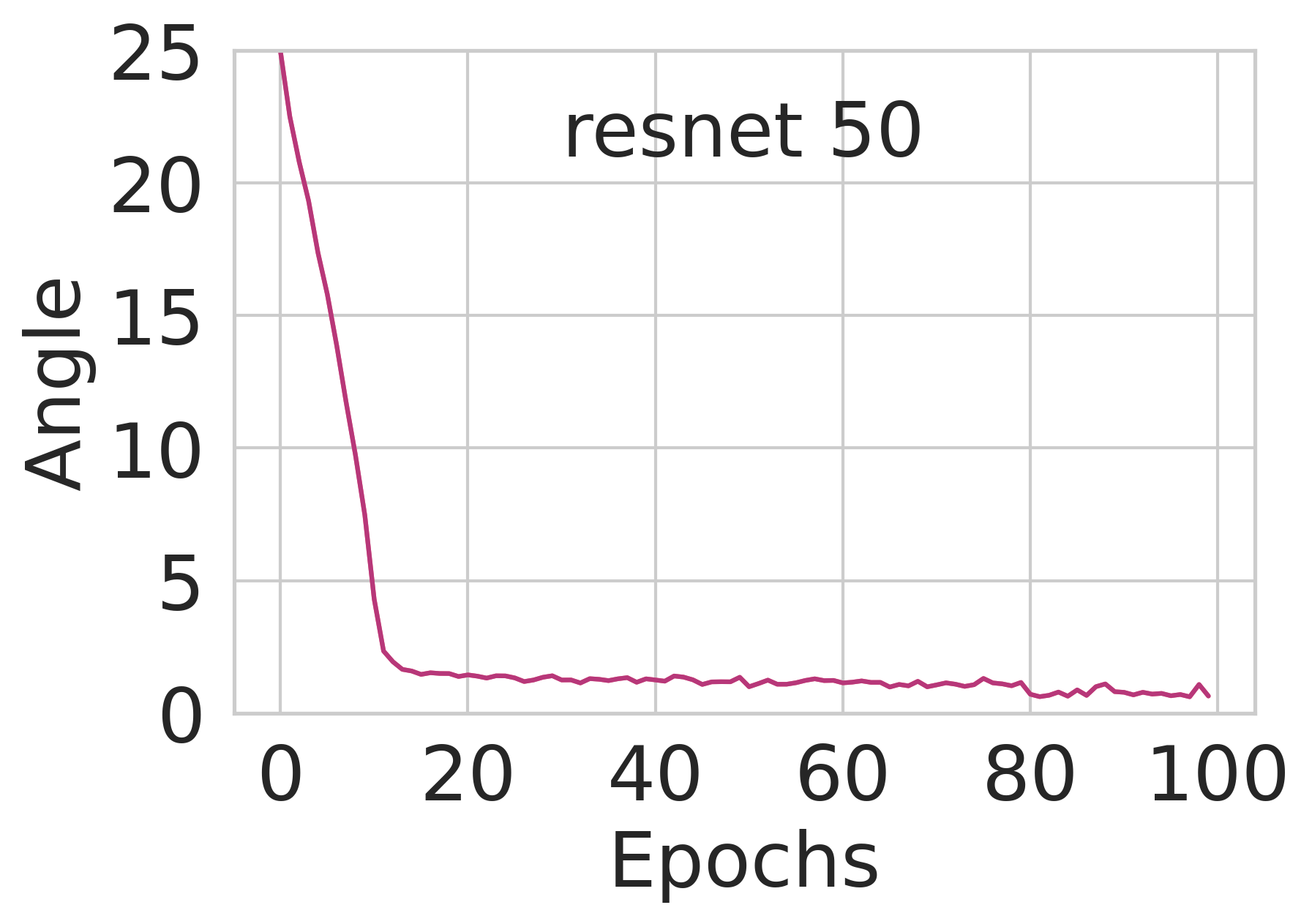}
  \includegraphics[width=.24\linewidth]{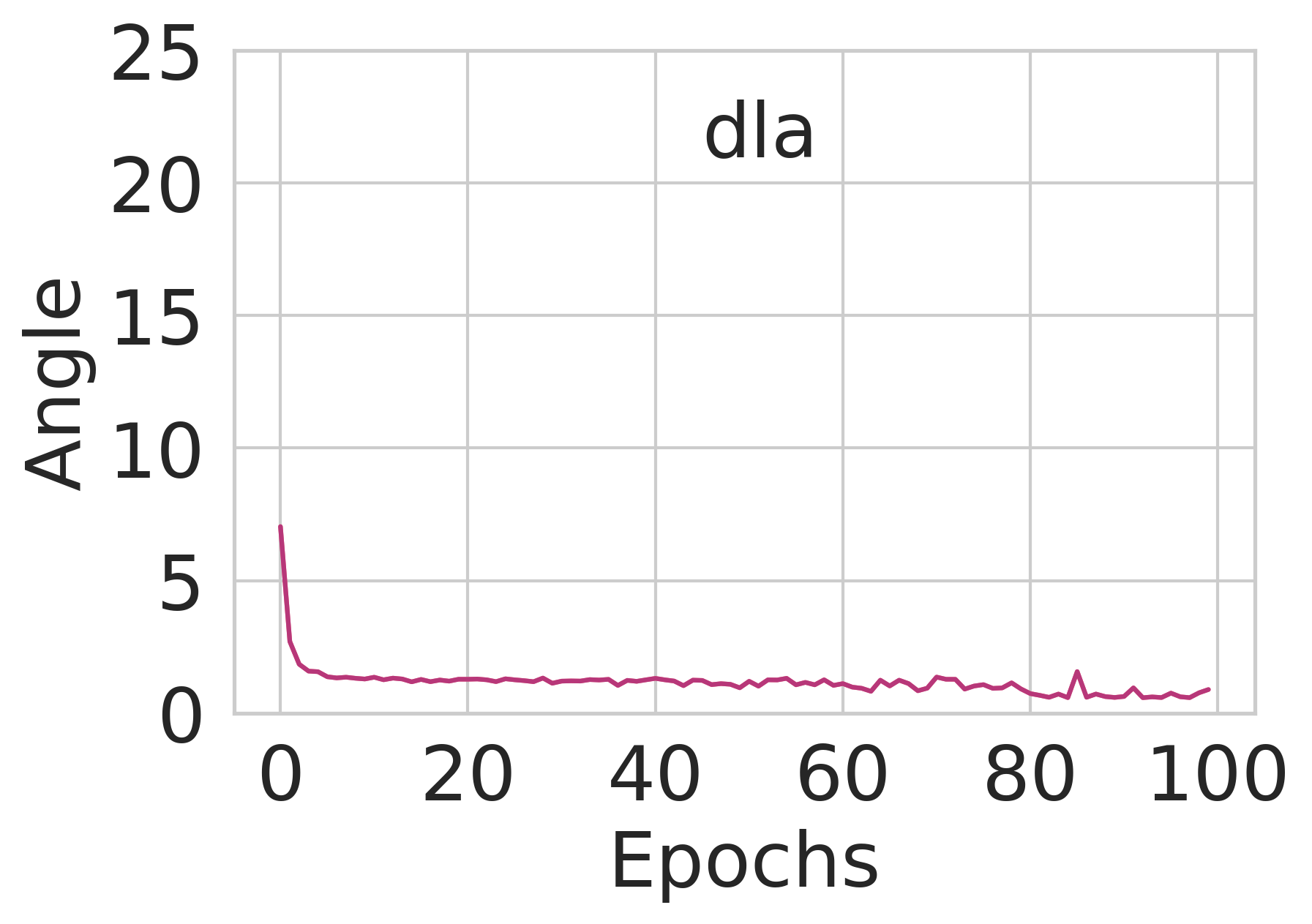}
  \includegraphics[width=.24\linewidth]{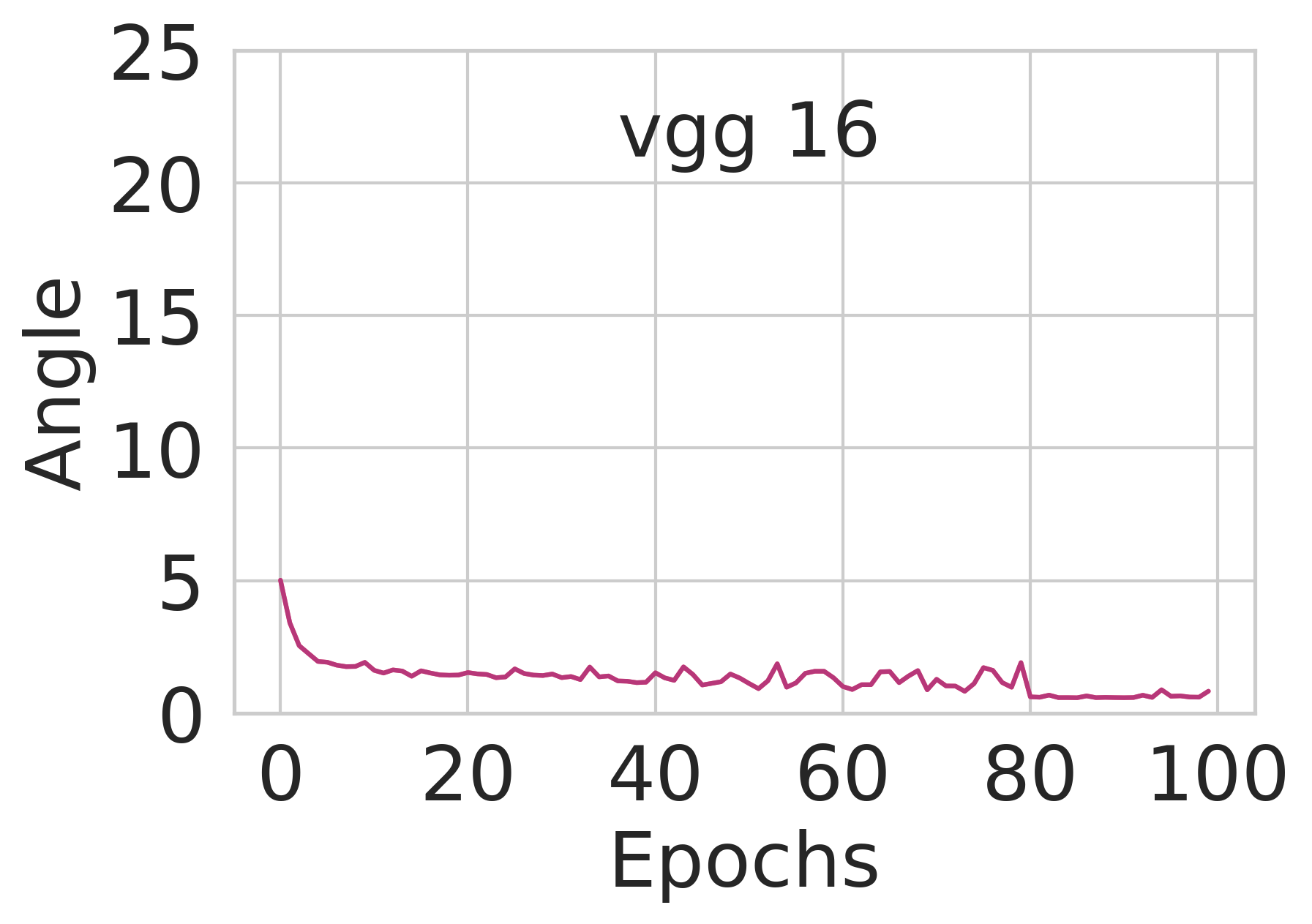}
  \caption{\label{fig:angle_results}Progression of angles with increasing epochs on CIFAR 10 dataset.  We observe that for most of the architectures, except of epochs less than 10, the angles are between 0-2 degrees. This is configured in such a way because too less of an angle and the cot becomes unstable, and attempting to have larger angles by having a larger $h$ can lead to  $g_1$ and $g_2$ (See Figure \ref{alg:algorithm}) in the algorithm \ref{alg:cap}, to lie in different localities.}
\end{figure*}%

\section{Introduction}
Stochastic gradient methods have become increasingly popular in various fields due to their effectiveness in optimizing large-scale problems and handling noisy data. In machine learning, they are extensively used to train deep learning models, such as neural networks, where they enable efficient optimization of high-dimensional, non-convex loss functions. Stochastic gradient descent (SGD) and its variants have been the cornerstone of training deep learning models in areas like computer vision, natural language processing \cite{GoodBengCour16,Mandlecha22,vaswani2017attention}, and reinforcement learning \cite{Sutton1998,mehta2023}. Additionally, stochastic optimization techniques have been adapted for online and distributed learning, facilitating real-time adaptation and scalability to massive datasets in applications like recommendation systems \cite{mishra2023,dasgupta2023review} and search engines. 

Although the popular stochastic gradient descent (SGD) is a simple method and is known to have a convergence to local minima with a small enough learning rate, it doesn't give better results compared to modern optimizers in deep learning architectures. The vanilla SGD uses the same learning rate for the entire training process and incorporates very little curvature information, leading it to get stuck to a local optimum. 

Recent approaches have tried to overcome the oscillations associated with SGD with the help of momentum \cite{sutskever2013}, authors in \cite{QIAN1999145} have argued that the reason momentum speeds up the convergence is that it brings some eigen components of the system closer to critical damping. 

\begin{figure}[]
    \centering
    \includegraphics[width=8.3cm]{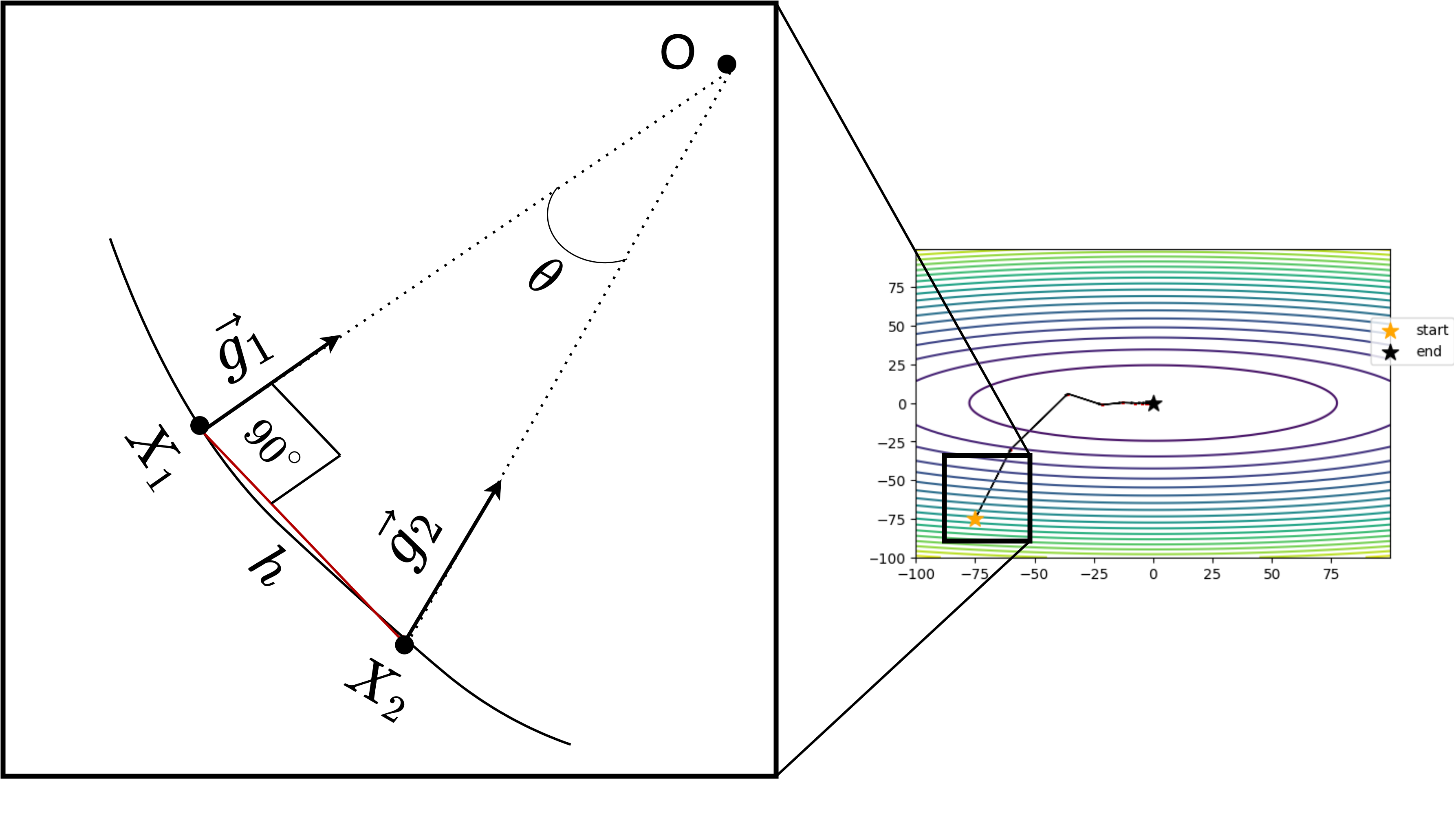}
    \caption{Illustration of one iteration of the proposed method.}
    \label{fig:method_diagram}
\end{figure}

Although momentum can deal with the oscillations in training, the problem of fixed learning rate is not addressed solely by momentum. There have been many works in the literature to address the static learning rate \cite{ruder2016overview}; the Rprop algorithm \cite{298623} had tried to accomplish a dynamic learning step for full batch gradient descent by increasing and decreasing the step size of a particular weight corresponding to the sign of the last two gradients. If the signs are opposite, then that indicates deviation, and hence the learning rate is decremented, and if the sign is positive, then the learning rate is incremented. Hence, 
the step size is adapted for every iteration.  

Another branch of optimizers like Nesterov Accelerated Gradient (NAG) \cite{nesterov1983method} include those that try to correct the gradient using the momentum to form a look-ahead vector. From this look-ahead, they see the direction of the gradient at that point, and subsequently, use that gradient to perform the weight updates from the previous point. 

Adagrad \cite{duchi2011adaptive} tries to use the expectation of the square of the gradients to scale the gradient element-wise, but it failed to learn after some epochs as the summation keep growing during the training. An extension to Adagrad is the algorithm Adadelta \cite{zeiler2012adadelta}; instead of just accumulating gradients, they use a decaying average of past updates, alongside exponentially decaying average of gradients. The method RMSProp \cite{hinton2012neural}, which was introduced in a lecture as a mini-batch version of RProp,  also uses an exponentially decaying average of gradients, but it differs from Adadelta because it doesn't use the exponentially decaying average of past updates; however, the first update of Adadelta is same as that of RMSProp.

Another recently popular optimizer Adam \cite{kingma2014adam}, can be thought of as a combination of momentum and RMSprop. Adam incorporates gradient scaling witnessed in RMSprop and use momentum via the moving average of gradients to tackle oscillations. It also includes bias corrections.

Authors in \cite{dubey2019diffgrad} introduce a friction coefficient called diffGrad friction coefficient (DFC); They compute this coefficient using a non-linear sigmoid function on the change in gradients of current and immediate past iteration. Adabelief \cite{zhuang2020adabelief}, another popular optimizer uses the square of the difference between the current gradient and the current exponential moving average to scale the current gradient. Many of the popular optimizers have played around with these four properties, namely, momentum, a decaying average of gradients or weight updates, gradient scaling, and a lookahead vector to derive their gradient update rules. 

There have been few papers that have considered the angle information between successive gradients. The authors in \cite{9343305} try to correct the obtained gradient by using the angle between the successive current and the past gradients. The authors in \cite{roy2021angulargrad} have also used the angle between successive gradients. They control the step size based on the angle obtained from previous iterations and use it to calculate the angular coefficient, essentially the {\tt tanh} of the angle between successive gradients. They claim that it reduces the oscillations between the successive gradient vectors; they also combine the gradient scaling, like RMSprop, and the expected weighted average of gradients as in ADAM.

Other notable optimization methods which are variants of Adam for deep learning are NADAM \cite{dozat2016incorporating}, Rectified Adam (RAdam) \cite{liu2019variance}, Nostalgic Adam \cite{huang2018nostalgic}, and AdamP \cite{heo2020adamp}. Preconditioning gradient methods are usually popular to accelerate the convergence of gradient methods \cite{saad2003iterative,Das_2021_WACV,9093265,das2020,kumar2013,kumar2014,kumar2013b,KUMAR20132251,BENZI2002418} for solving large sparse linear systems stemming from scientific computing applications. Similar preconditioned stochastic gradient methods have been tried \cite{stefan2011,david2015}, but they are not as effective.

In this paper, we propose a simple technique for estimating an adaptive learning rate for SGD. The learning rate is obtained by determining the angle between the current gradient $g_1$ computed at the current location $x_1$ and a new probing gradient $g_2$ computed at a new location $x_2,$ which is at an orthogonal distance from the current location. This angle is then used to determine an effective learning rate (See Algorithm \ref{alg:algorithm} for details). A graphical illustration of our method is shown in Figure \ref{fig:method_diagram}. 

\subsection{Our Contributions}
We summarize our contributions as follows:
\begin{enumerate}
    \item We propose a dynamic learning rate for vanilla SGD. Our method uses a geometric approach, which is simple, easy to implement, and very effective.
    \item We show that under certain assumptions the proposed method is convergent and satisfies the well-known Armijo's condition for sufficient decrease. 
    \item We show an extensive set of accuracy results and comparisons of prominent stochastic optimization solvers on various image classification architectures and prominent datasets.
\end{enumerate}

\section{The Proposed Method}

To facilitate an adaptive and more accurate step size, we consider a snapshot in the global gradient vector field, it is well known that the gradient vectors will be perpendicular to the tangential plane of the level curve at each point of the vector field. The existing gradient descent methods go parallel to the negative gradient $g_1$ from the point $X_1,$ instead, we first go perpendicular in the direction of $ g_1{^{\perp}}$ towards a new point $X_2$ by a small $ h$ step (a hyperparameter); the direction $ g_1{^{\perp}}$ is any random vector inside the tangential plane. At $X_2$ we calculate another negative gradient $g_2$. Let point $O$ be the intersection point of these two gradients. We now simply use these information to calculate the height of the triangle as shown in the Figure \ref{fig:method_diagram} by calculating the step size $d = h \cdot \cot{(\theta)}$ in the step 8 of the Algorithm \ref{alg:cap}, where $\theta$ is the angle between the two vectors $g_1$ and $g_2$. We then travel in the direction of normalized $g_1$ with step size $d$ as shown in step 13 of the algorithm \ref{alg:algorithm}, and essentially reach the intersection point $O$ of the two gradients; we repeat this process until convergence.

Although the perpendicular step size is fixed and static, our learning rate changes with every new iteration. The perpendicular step $h$ plays a critical role in our method; too large of a perpendicular step will lead $g_1$ and $g_2$ to point to different localities in non-convex settings. Hence, the calculation of step length will make no sense. On the other hand, if the value of $h$ is too small, then the gradient vectors will overlap, and the angle between these gradients will be too small to compute precisely. So to tackle this issue, we keep a relatively small perpendicular step size $h$, and we add some $\epsilon$ to $h$ because $\cot{\theta}$ is infinite when $\theta$ is zero. Figure \ref{fig:angle_results} shows the resulting angle for the optimal $h$ on different architectures. In steps 8 to 11 of algorithm \ref{alg:cap}, we use the decayed expectation (step 9) of the step size to double the current estimated step size (step 11), if the current estimated step size is lesser than the expected step size.


\begin{algorithm}[h!]
\caption{\label{alg:cap}Our Method}
\label{alg:algorithm}
\textbf{Input}: $X$\\
\textbf{Parameter}: $h,$ $\beta$
\begin{algorithmic}[1] 
\While {not converged}
\State $ X_1 \gets X $ \Comment{Initial random guess}
\State $ g_1 \gets -\nabla_X f(X_1)$
\State $ p_1 \gets g_1{^{\perp}} $ \Comment{Perpendicular unit vector to $g_1$}
\State $ X_2 \leftarrow X_1 - h \cdot p_1 $
\State $ g_2 \gets -\nabla_X f(X_2) $
\State $ \theta \gets \angle (g_1, g_2) + \epsilon $
\State $ d \gets  h \cdot \cot{\theta}$ \Comment{Find current step size}
\State $d_{avg} \gets \beta \cdot d_{avg} + (1-\beta) \cdot d$ \Comment{Average step size}
\If{$d < d_{avg}$}
    \State $ d \gets 2 \cdot d$
\EndIf
\State $ X \gets X_1 + d \cdot \dfrac{g_1}{\| g_1 \|} $ \Comment{Update step}
\EndWhile
\end{algorithmic}
\end{algorithm}

\begin{figure*}[ht]
\begin{center}
\begin{subfigure}[b]{\linewidth} 
\center
  \includegraphics[width=.23\linewidth]{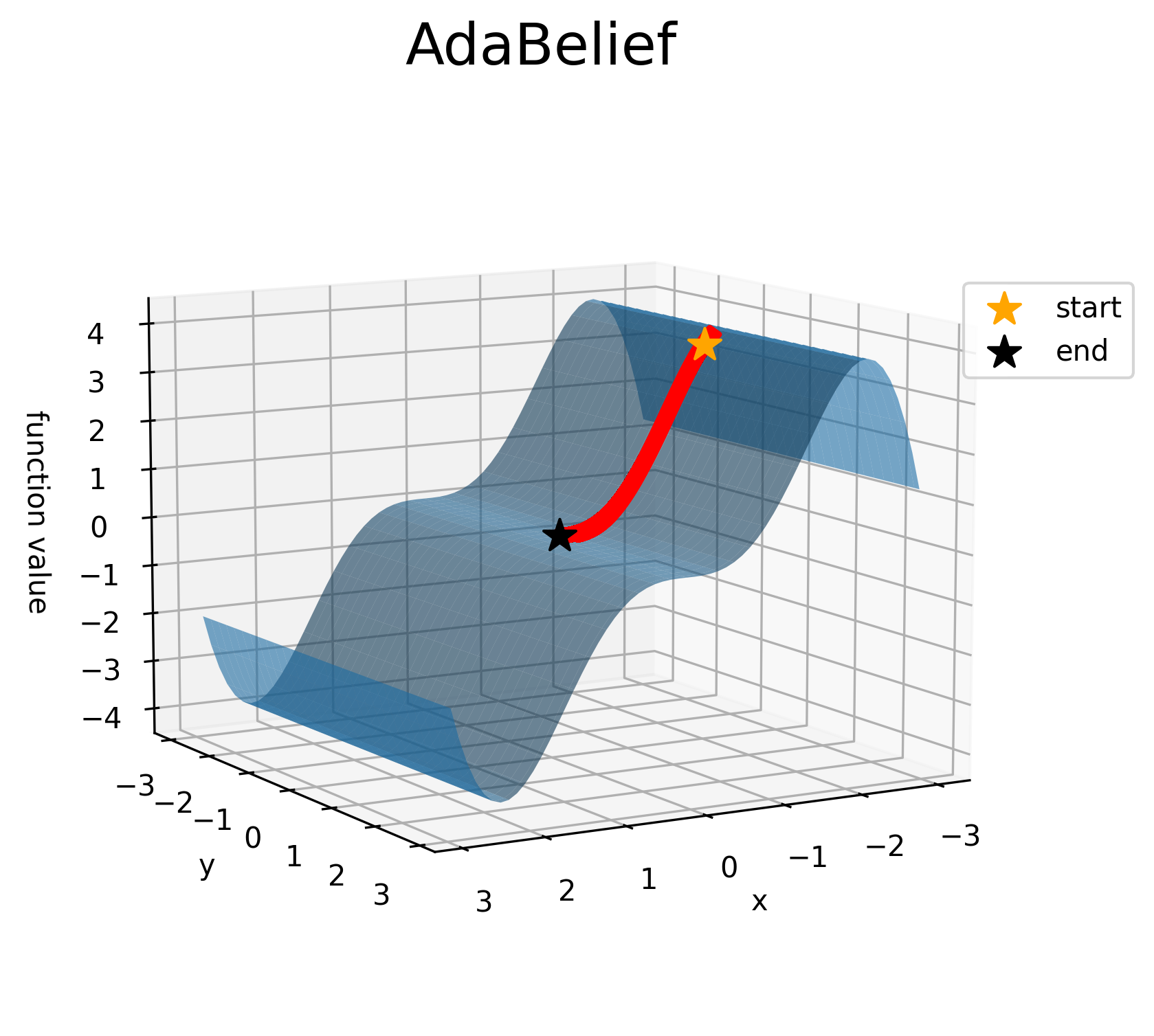}
  \includegraphics[width=.23\linewidth]{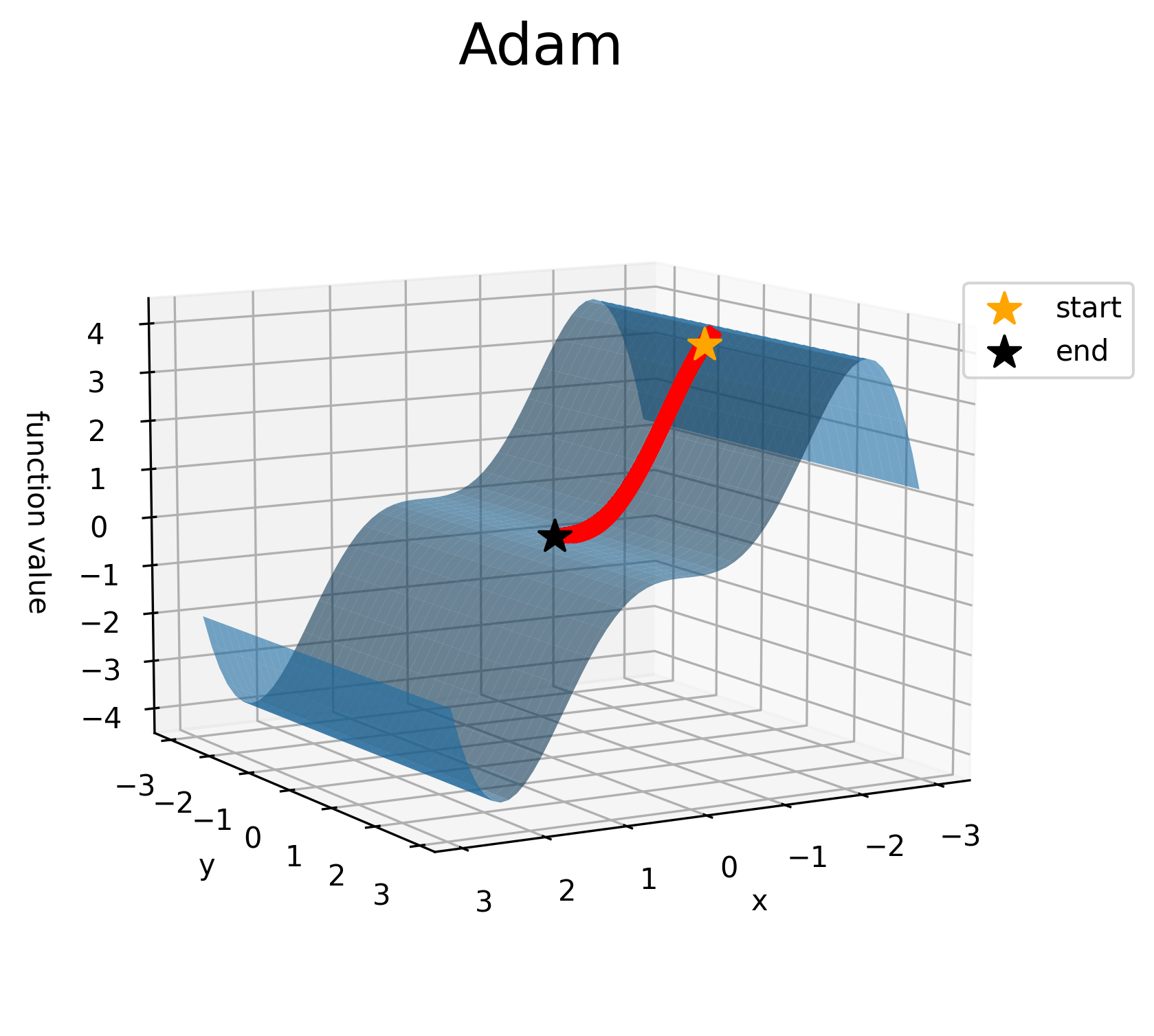}
  \includegraphics[width=.23\linewidth]{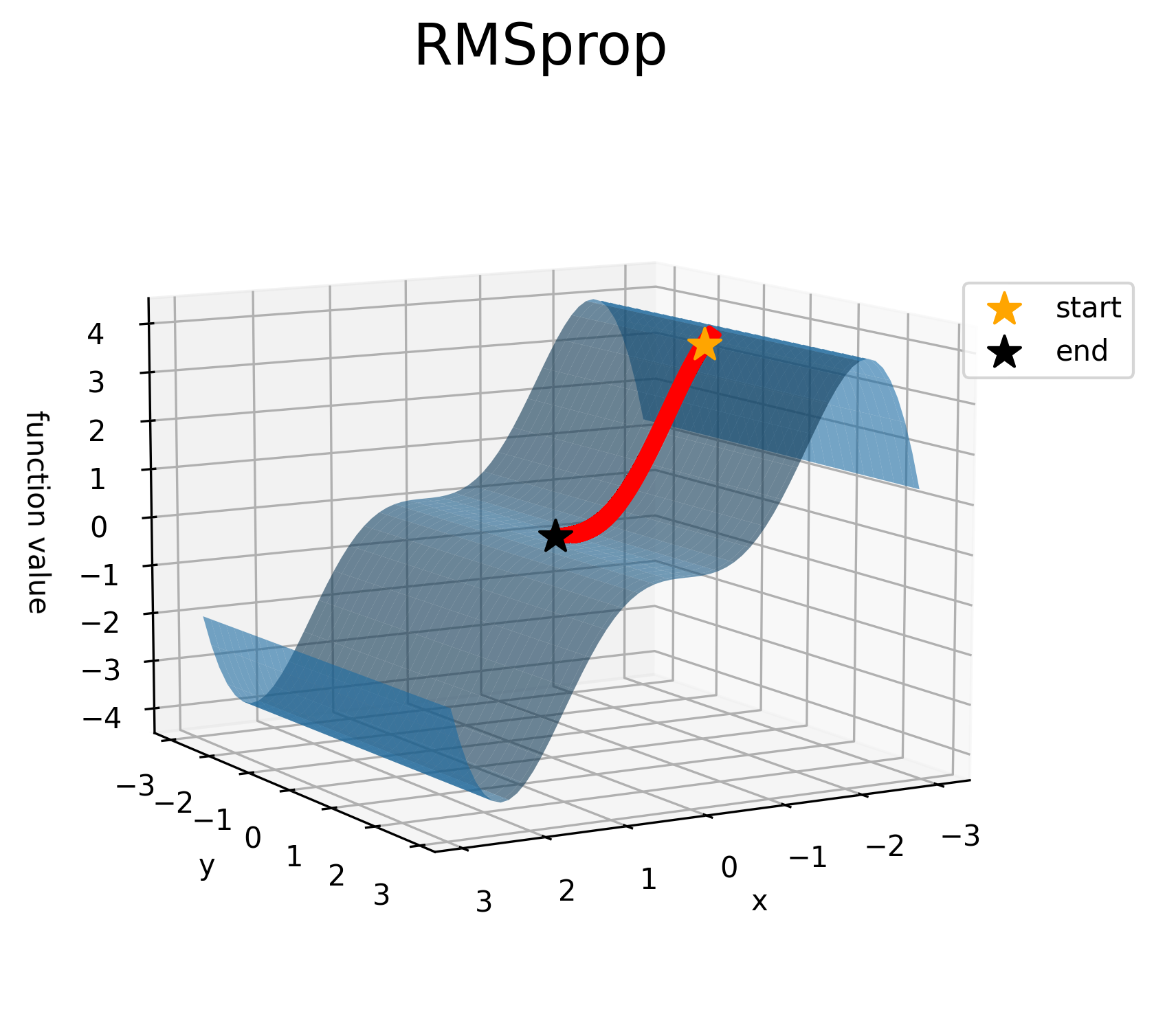}
  \includegraphics[width=.23\linewidth]{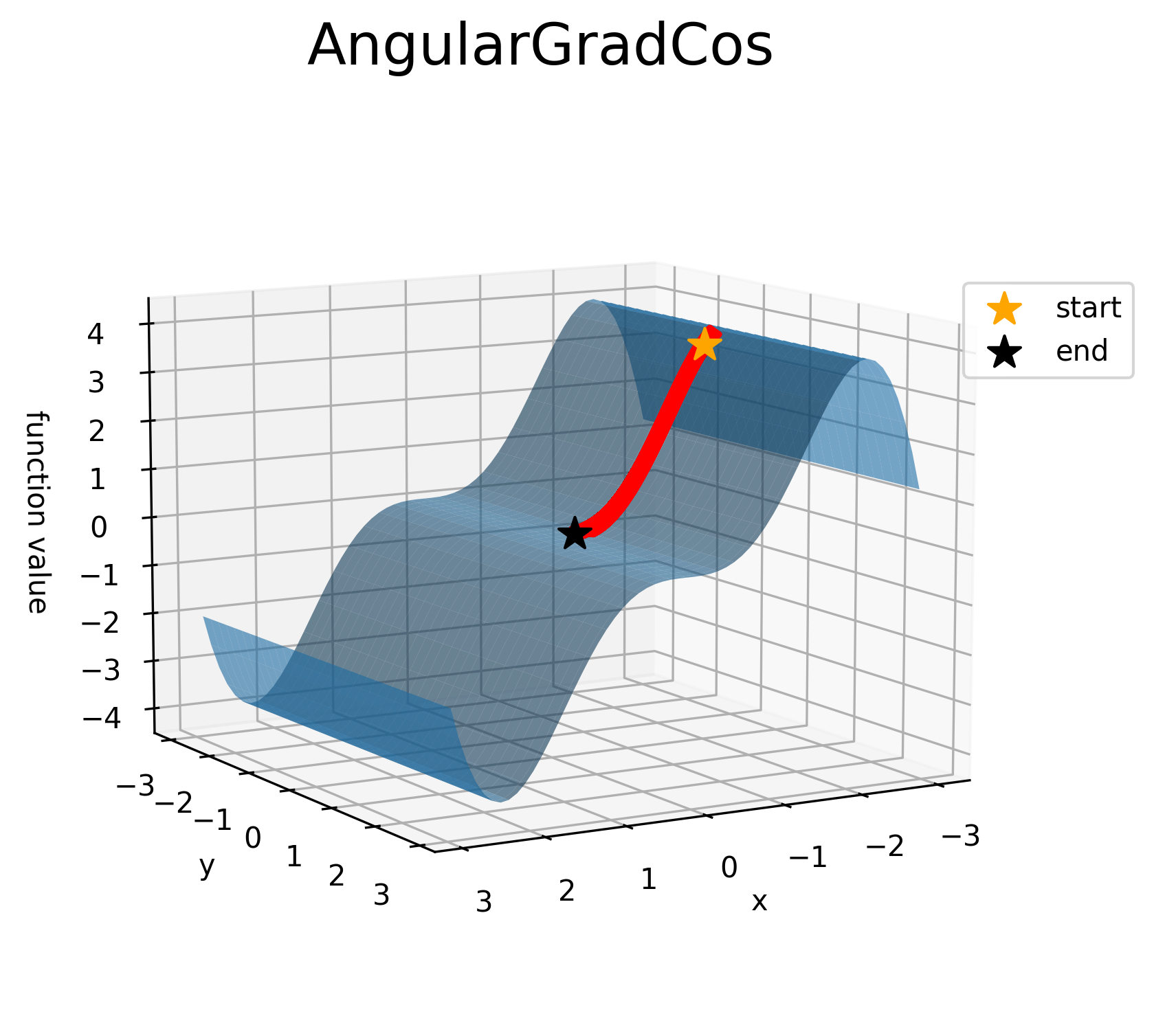}
  \includegraphics[width=.23\linewidth]{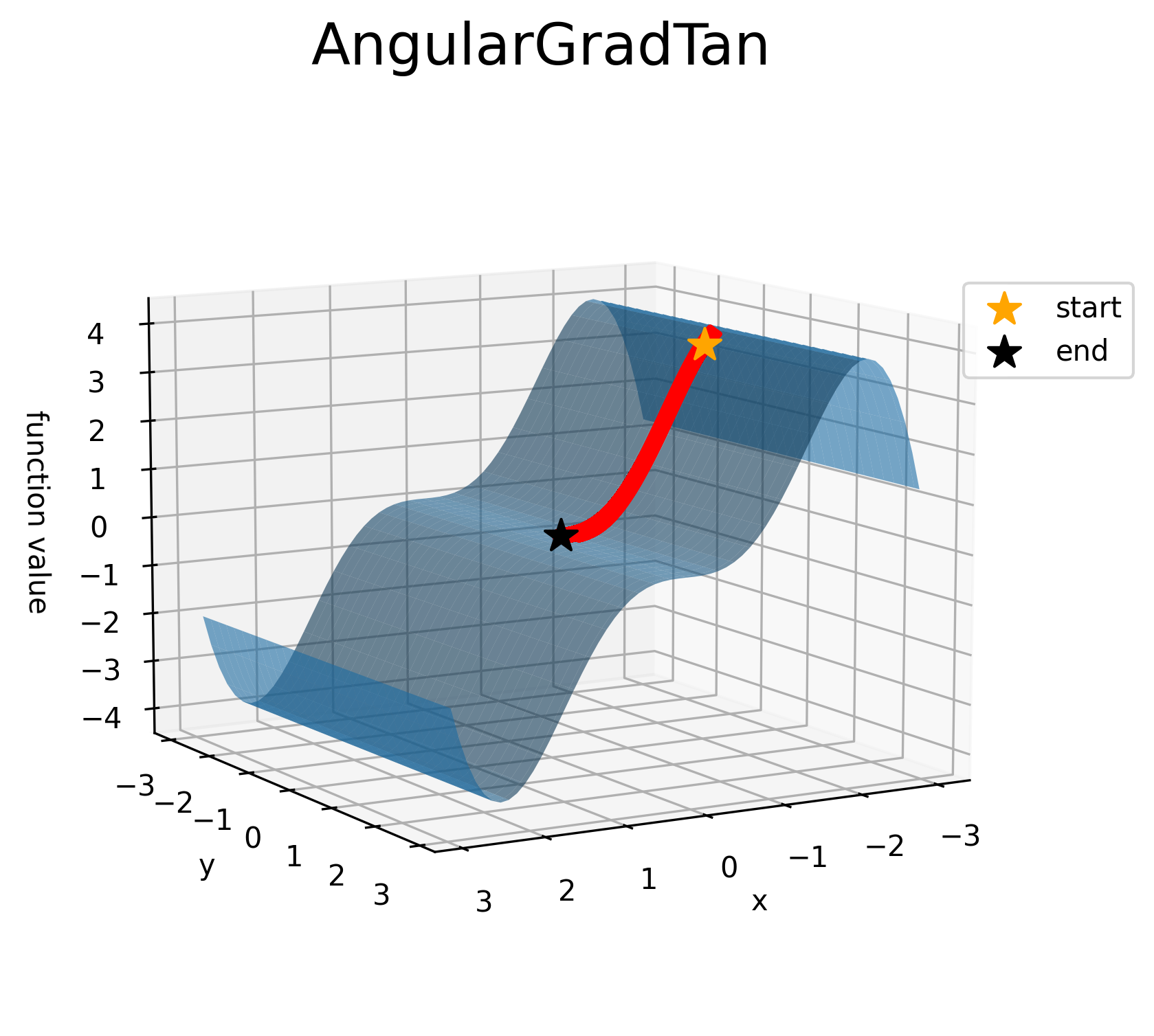}
  \includegraphics[width=.23\linewidth]{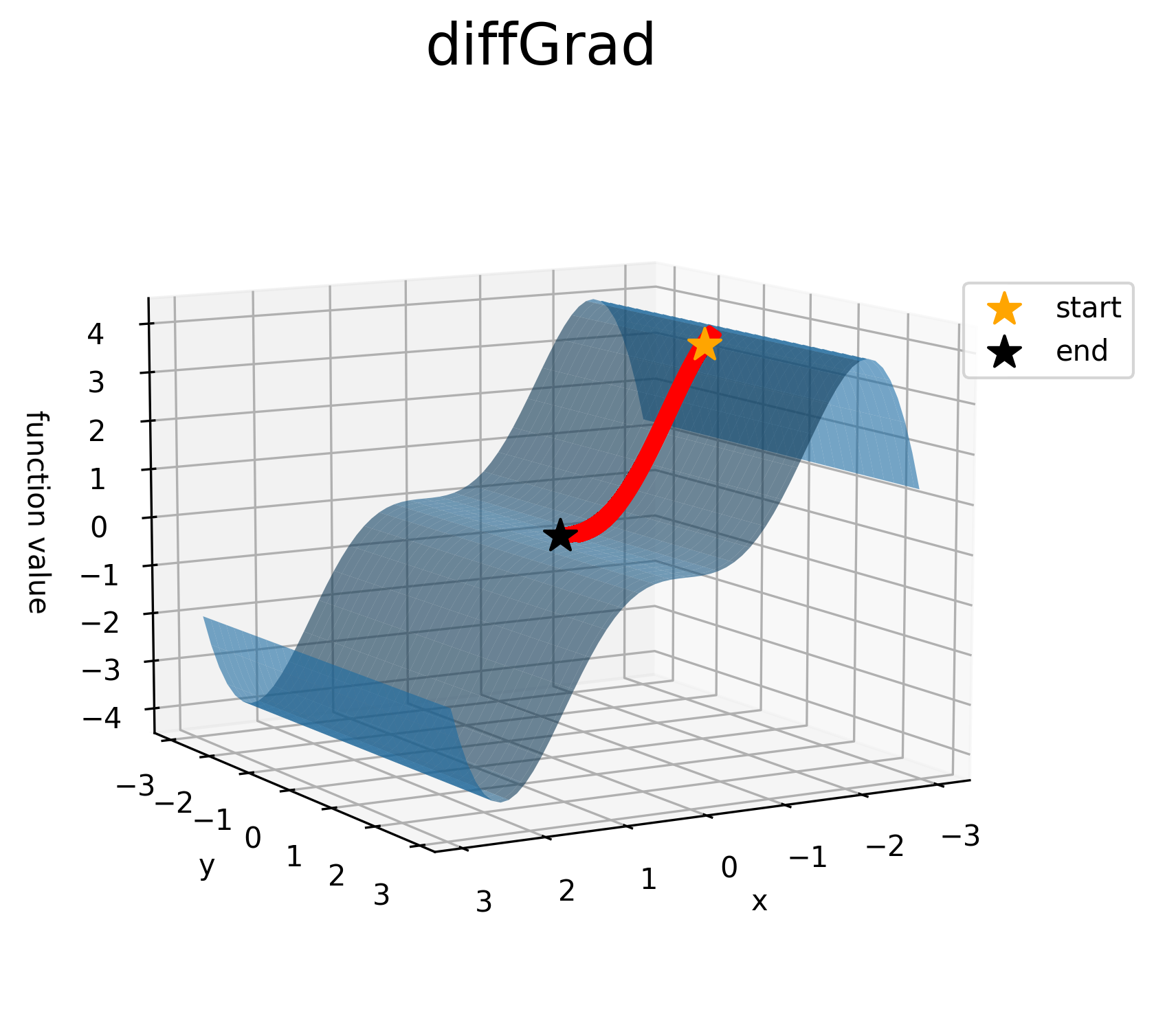}
  \includegraphics[width=.23\linewidth]{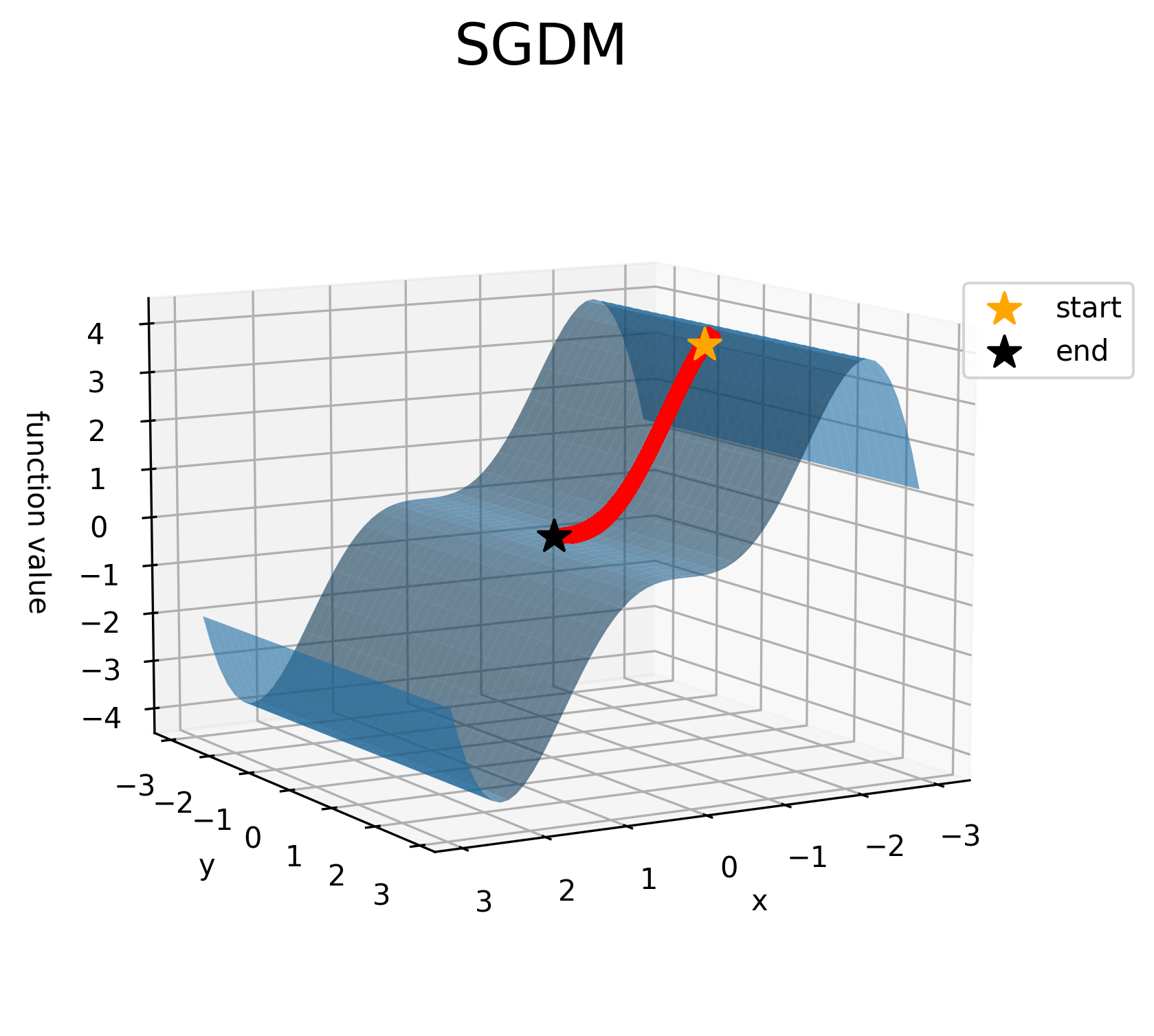}
  \includegraphics[width=.23\linewidth]{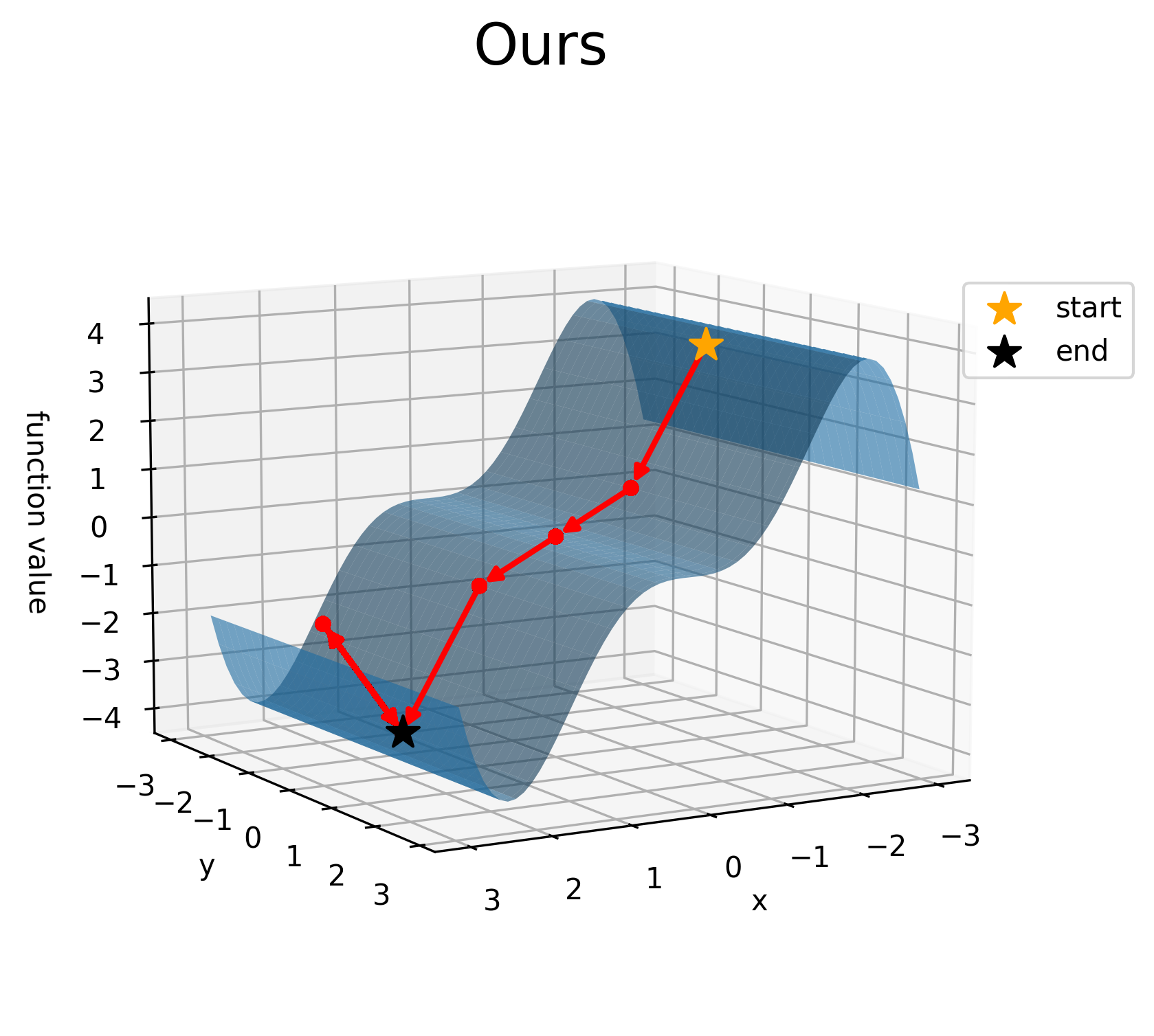}
  \caption{Toy Example A: Our method is shown as bottom right figure. We observe that while all other methods get stuck in the local minima, our method can descents further. The black star denotes the terminal point, while the yellow star denotes the starting point.}
  \label{subfigure:toy_example_1}
\end{subfigure}
\end{center}
\end{figure*}

  

\section{Experiments on Some Toy Examples}
In this section, we experimentally demonstrate the potential of our method to overcome bad minima; for this, we have crafted  two challenging surfaces using the tool \cite{BenJoffe}. We have also shown the 3D visualization for the same in Figures \ref{subfigure:toy_example_1} and \ref{subfigure:toy_example_2}. For the experimental setup on these toy examples, we set the learning rate $\eta=1e-2$, we keep the same initial point for all the methods, and we train them up to 1000 iterations. 

\subsection{Toy Example A}
In the toy example $A,$ we consider the following simple function 
\begin{align}
 f(x,y) = -y^2\sin{x}.
\label{eqn:toyA}   
\end{align}
We see in Figure \ref{subfigure:toy_example_1} that the surface corresponding to this function is wavy and has multiple local minima. The initialization point is $(-2.0,0.0)$ as shown in the Figure \ref{subfigure:toy_example_1}. Although our trajectory is the same as that of SGD without momentum, we can see that while all the optimizers get stuck at the relatively bad local minima, our proposed optimizer goes to a better local minimum, and it doesn't stop there; we can see that it tries to go further only to come back due to steep curvature ahead. Even though the gradient amplitude may be small in a region, since we take the unit direction of gradient to estimate the perpendicular direction, this perpendicular step is invariant to the current gradient amplitude. Hence, it can explore a region where the gradient is more substantial and use it to estimate an effective step length that may lead to further descent.

\begin{figure*}[ht]
\begin{center}
\begin{subfigure}[b]{\linewidth} 
  \includegraphics[width=.23\linewidth]{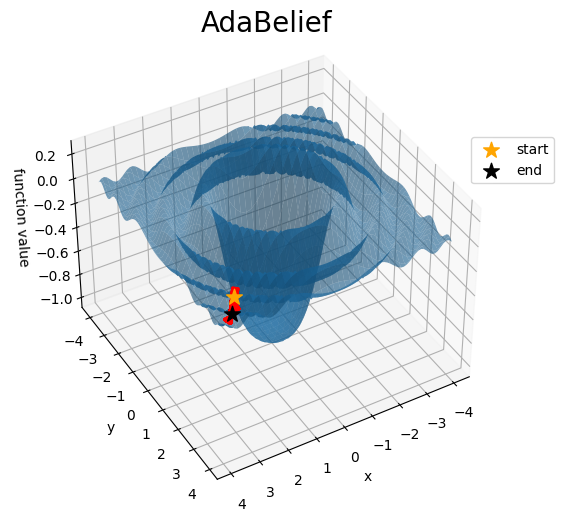}
  \includegraphics[width=.23\linewidth]{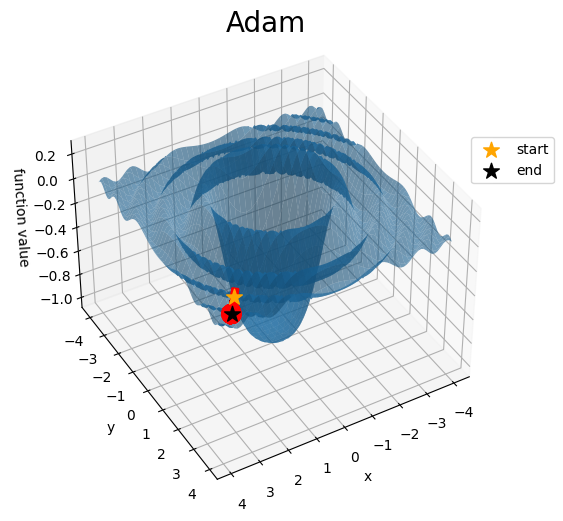}
  \includegraphics[width=.23\linewidth]{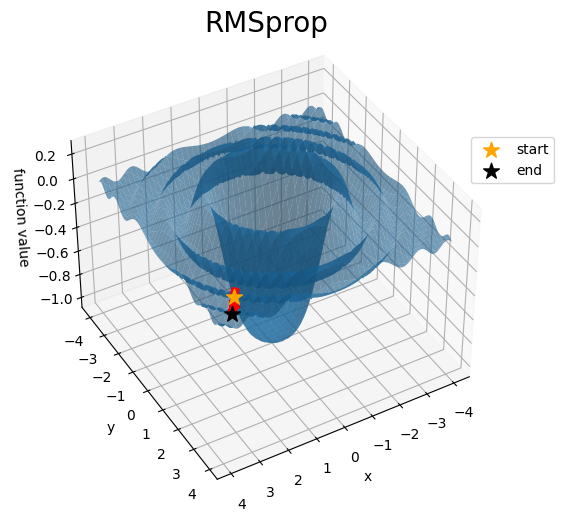}
  \includegraphics[width=.23\linewidth]{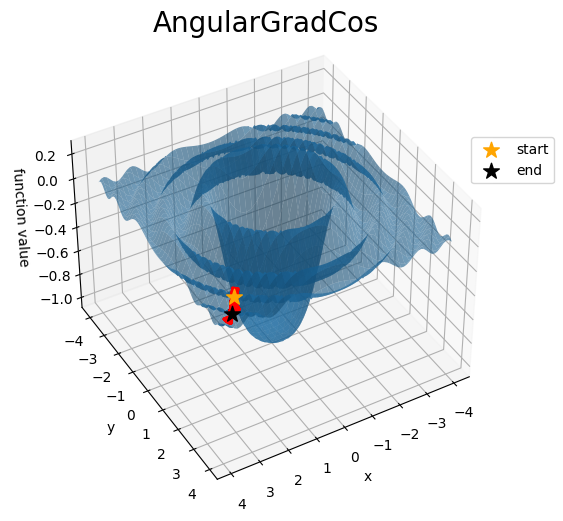}
  
  \includegraphics[width=.23\linewidth]{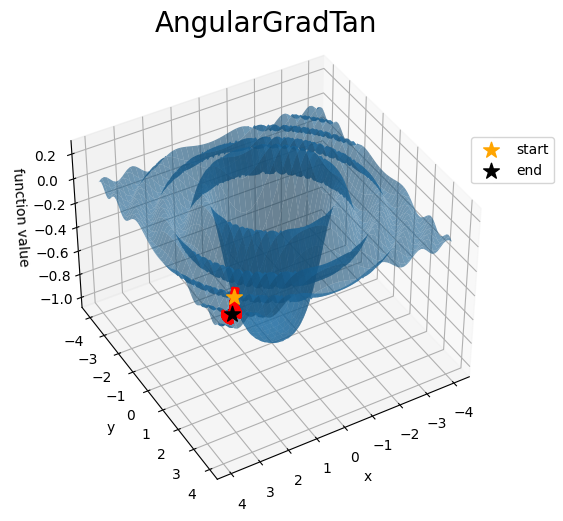}
  \includegraphics[width=.23\linewidth]{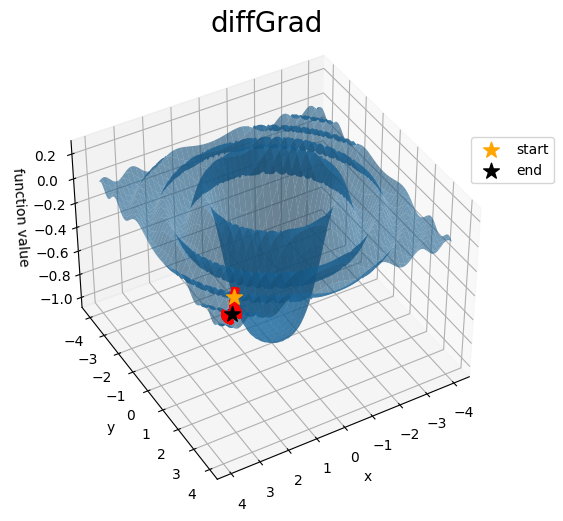}
  \includegraphics[width=.23\linewidth]{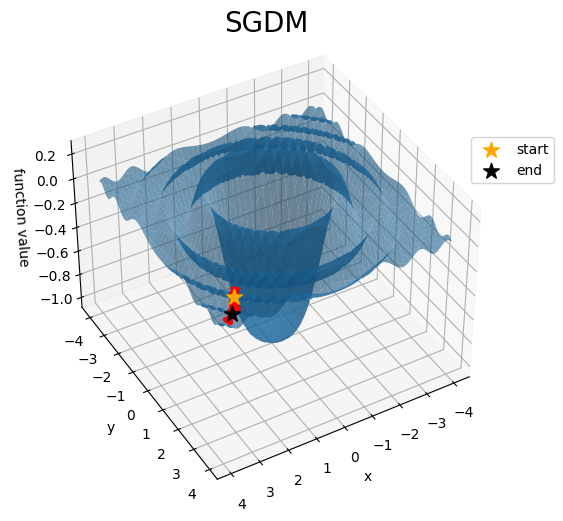}
  \includegraphics[width=.23\linewidth]{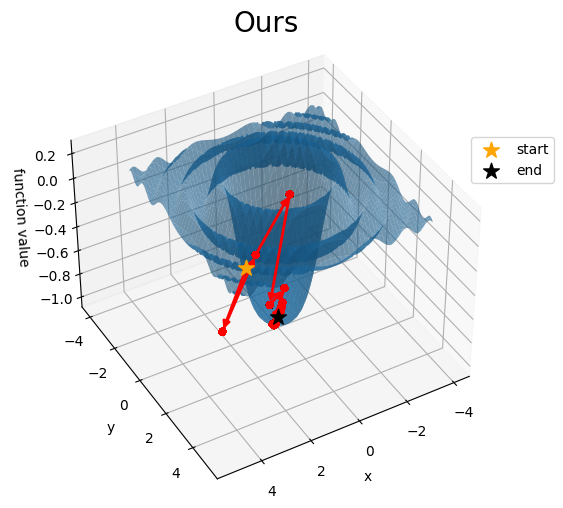}
  \caption{\label{subfigure:toy_example_2}Toy Example B: Here our method is shown in the bottom right. The black star denotes the terminal point, while the yellow star denotes the starting point. We observe that while all other methods cannot explore efficiently, our method reaches a better local minima. }
  \end{subfigure}
\end{center}
\end{figure*}

\subsection{Toy Example $B$}
The toy example $B$ is defined as follows 
\begin{align}
f(x,y) = -\frac{\sin{x^2 + y^2}}{x^2+y^2}.
\label{eqn:toyB}
\end{align}
The shape of $f(x,y)$ as shown in Figure \ref{subfigure:toy_example_2} is similar to a ripple in the water, the initialization point is kept at $(3.0,3.0).$ We observe that other optimizers get stuck within the ripple of the function, while ours can explore further and eventually find a better minimum. 

\normalsize

\begin{figure*} [ht]
  \centering
  \includegraphics[width=.24\linewidth]{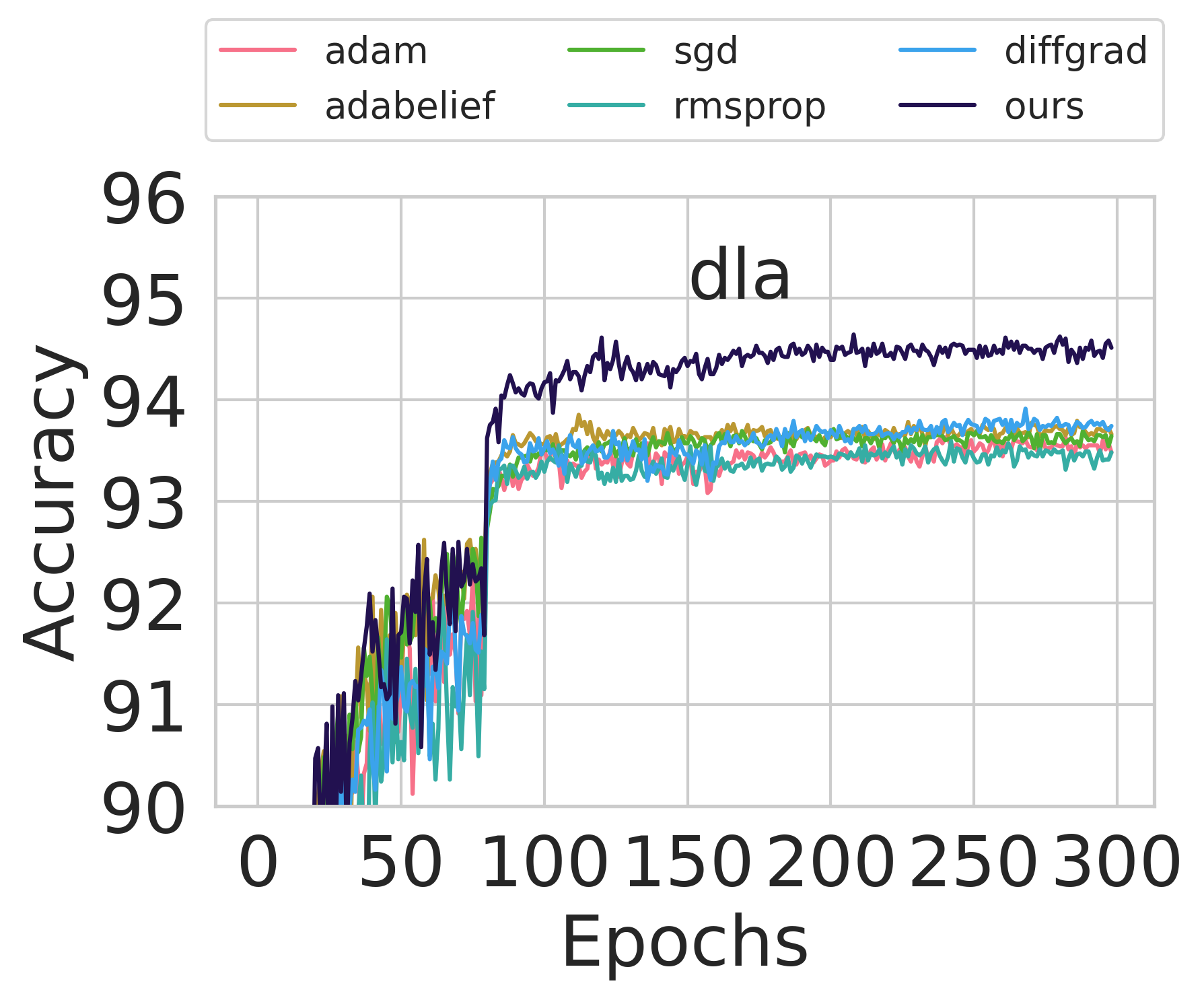}
  \includegraphics[width=.24\linewidth]{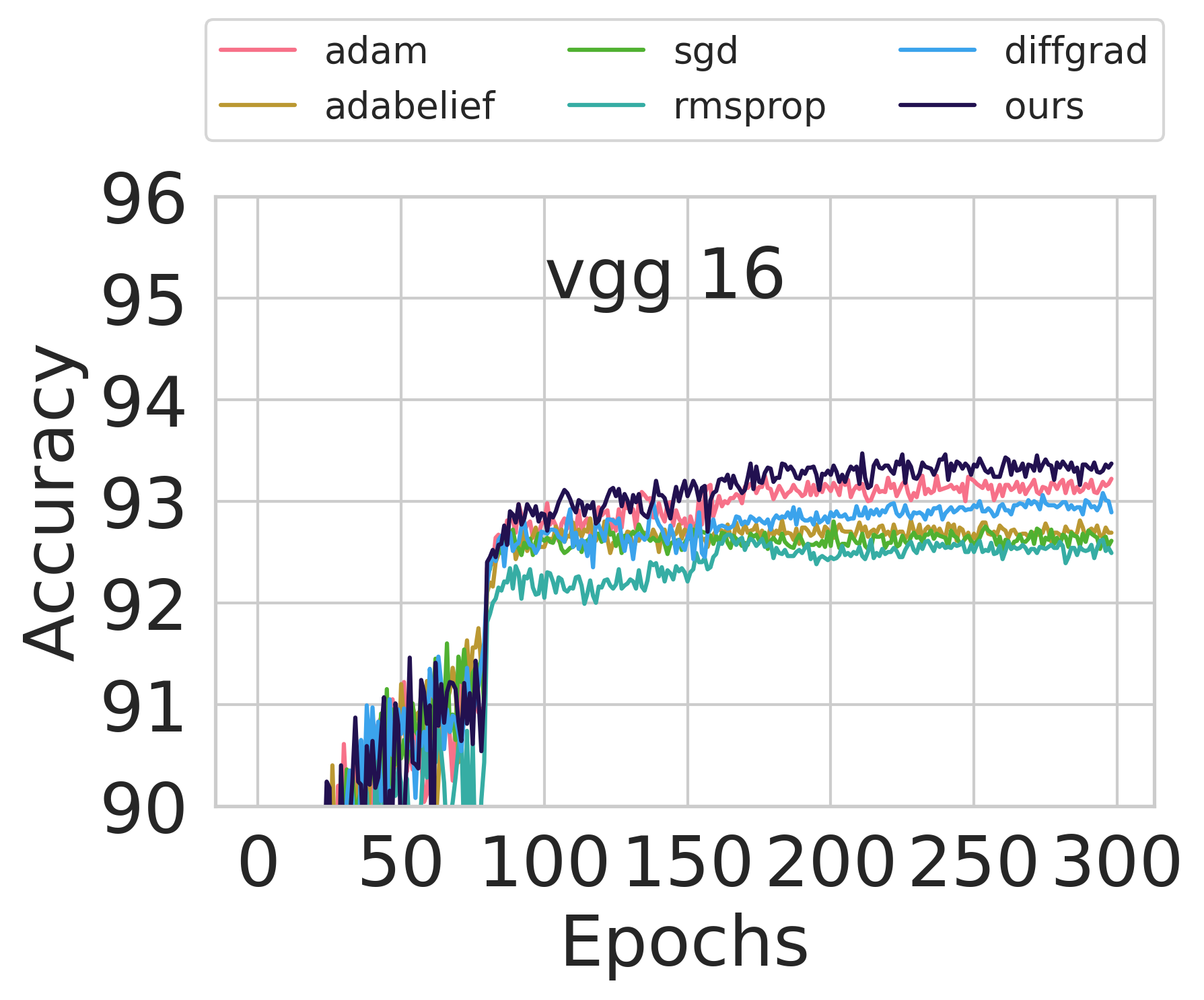}
  \includegraphics[width=.24\linewidth]{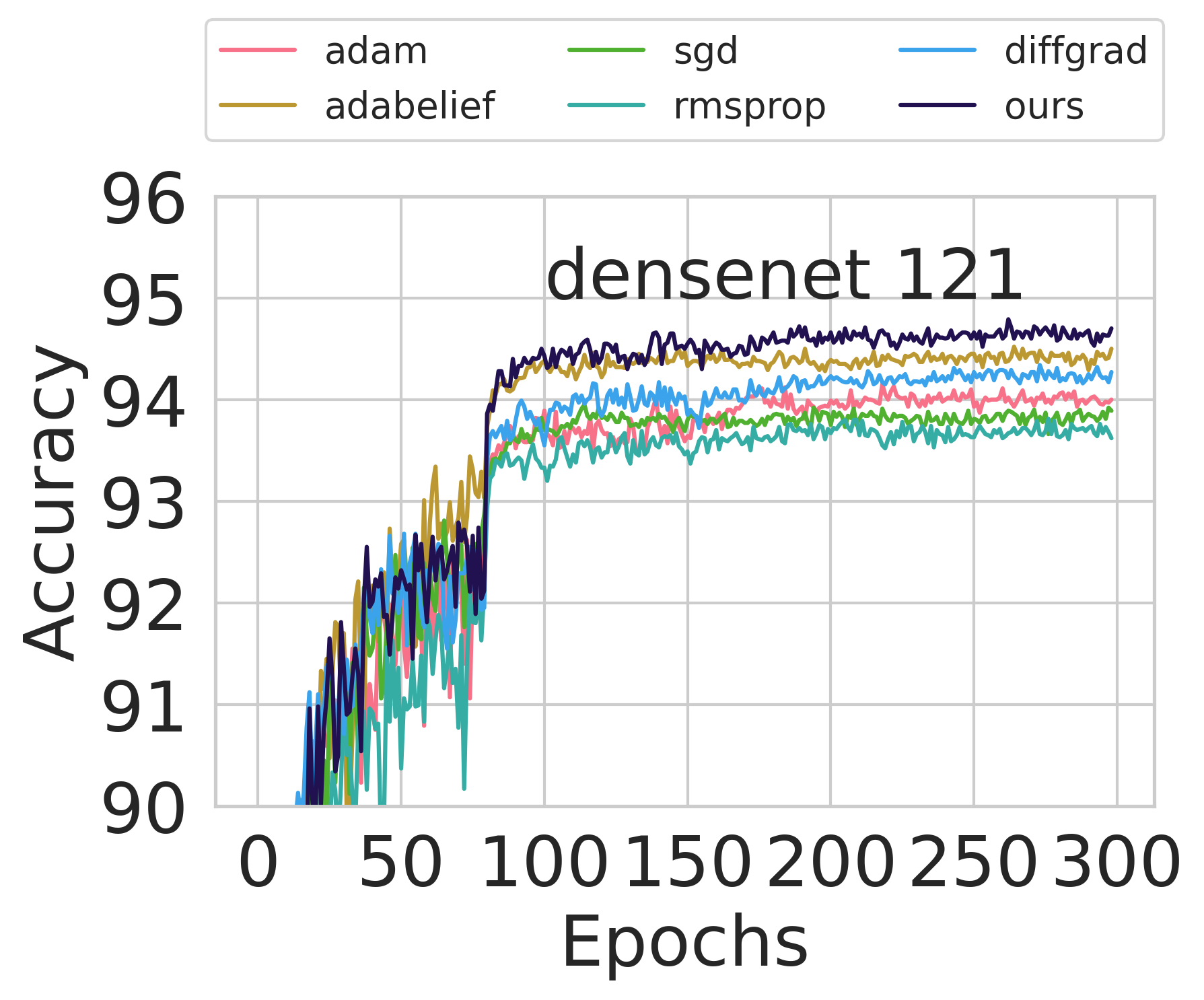}
  \includegraphics[width=.24\linewidth]{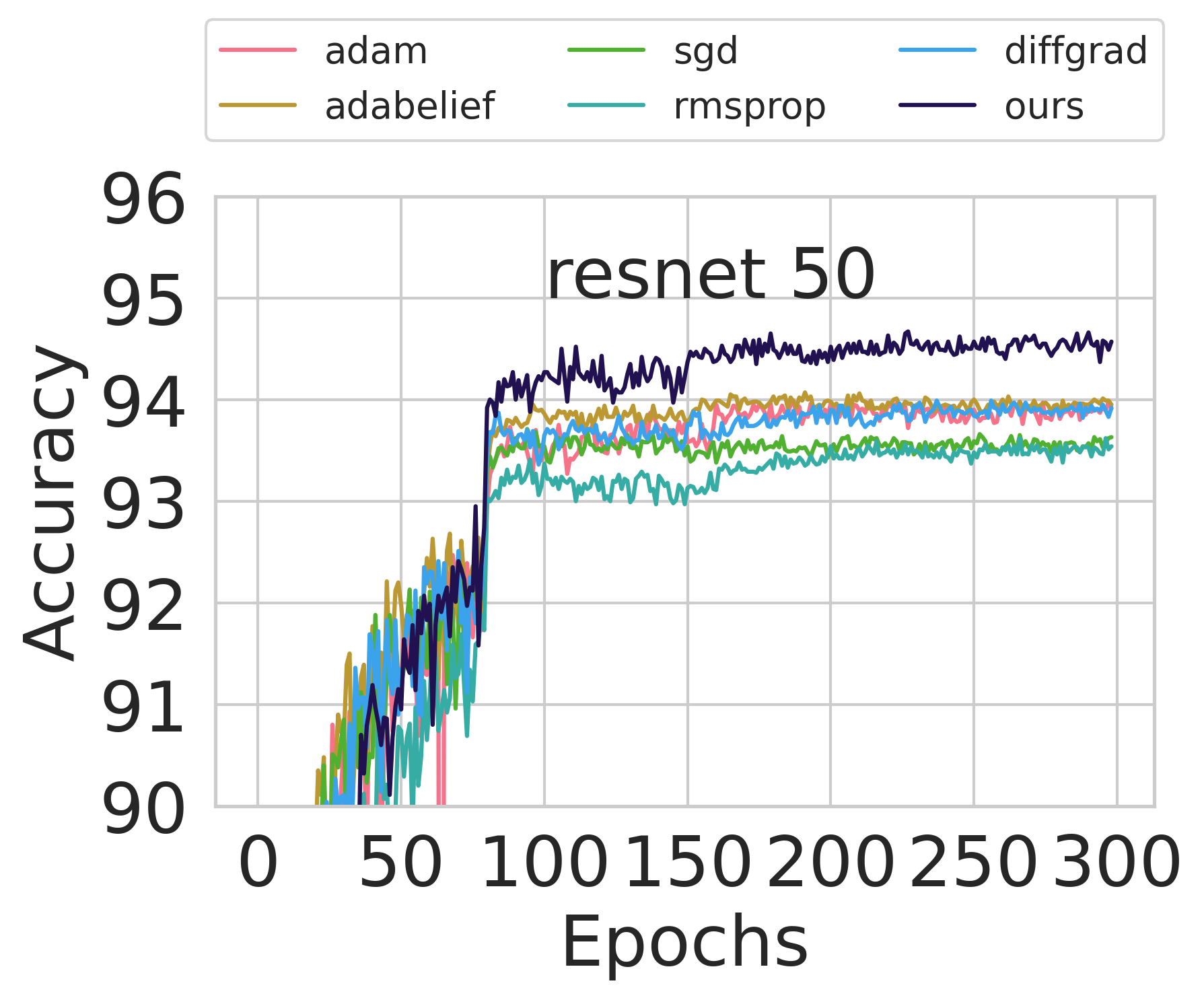}
  \caption{Accuracy versus epochs plot for image classification for CIFAR-100 results for 300 epochs. We verify that on long term our method maintains high accuracy, and most methods start to stagnate from epochs 100 onwards.}
  \label{fig:long_run}
\end{figure*}
\section{Convergence Results}
In this section, we prove convergence results. We show that with our choice of learning rate for SGD, our method converges. Moreover, later we also show that the Armijo's condition for our choice learning rate is satisfied.

\begin{theorem}
\label{th:1}
Let $f: \mathbb{R}^n  \rightarrow \mathbb{R}$  be a function which is convex and differentiable, and let the gradient of the function be $L$ Lipschitz continuous with Lipschitz constant $L > 0,$ i.e., we have that $\|\nabla f(x) - \nabla f(y)\| \le L \| x - y\| $ then the objective function value will be monotonously decreasing with each iteration of our method.
\end{theorem}

\begin{proof}
The below equation gives the update rule of our method
\small 
\begin{align}
\label{eq:update_rule}
    y = x_{t} = x_{t-1} - h_{t-1}\cot{\theta_{t-1}}\frac{\nabla f(x_{t-1})}{\|\nabla f(x_{t-1})\|}.
\end{align}
\normalsize
\text{And, let $h$ be constrained as,}

\small 
\begin{align}
    \label{eq:h_constrain}
    h_{t}\le\frac{\|f(x_{t})\|}{ L\cot{\theta_{t}}}.
\end{align}
\normalsize
With assumptions on $f(x),$ we have
\small 
\begin{align*}
f(y) &\le f(x) + \nabla f(x)^{T}(y-x) + \frac{1}{2}\nabla^2f(x)\|y-x\|^2. 
\end{align*}
\normalsize
\text{Since $\nabla f(x)$ is $L$-Lipschitz continuous,} $\nabla^2 f(x) \le L$
\small 
\begin{align*}
f(y) &\le f(x) + \nabla f(x)^{T}(y-x) + \frac{1}{2}L\|y-x\|^2. 
\end{align*}
\normalsize
Substituting $y$ as in equation \ref{eq:update_rule}, and $x$ as $x_{t-1}$, we get:
\small 
\begin{align*}
    f(x_{t}) &\le f(x_{t-1}) + \nabla f(x_{t-1})^{T}(- h_{t-1}\cot{\theta_{t-1}}\frac{\nabla f(x_{t-1})}{\|\nabla f(x_{t-1})\|}) \\
    &+ \frac{1}{2}L\|- h_{t-1}\cot{\theta_{t-1}}\frac{\nabla f(x_{t-1})}{\|\nabla f(x_{t-1})\|}\|^2 )\\
     &\le f(x_{t-1}) + \nabla f(x_{t-1})^{T}(- h_{t-1}\cot{\theta_{t-1}}\frac{\nabla f(x_{t-1})}{\|\nabla f(x_{t-1})\|}) \\
     &+ \frac{1}{2}L h_{t-1}^2\cot^2{\theta_{t-1}}) \\
     &\le f(x_{t-1}) + h_{t-1}\cot{\theta_{t-1}}(\frac{L}{2} h_{t-1}\cot{\theta_{t-1}}
     -\|\nabla f(x_{t-1})\|). 
\end{align*}
\normalsize

\text{Using equation \ref{eq:h_constrain}, we get}
\small 
\begin{align*}
    f(x_{t}) &\le f(x_{t-1}) + \frac{\|\nabla f(x_{t-1})\|}{L}( \frac{\|\nabla f(x_{t-1})\|}{2} 
    - \|\nabla f(x_{t-1})\|)  \\
     &\le f(x_{t-1}) - \frac{1}{2L}(\|\nabla f(x_{t-1})\|^2).  
\end{align*}
\normalsize
Hence the objective function decreases with every iterate.
\end{proof}
A well-known condition on the step length for sufficient decrease in function value is given by Wolfe's condition as stated below. We will show that the first Wolfe's condition, which is Armijo's condition is satisfied, however, we show that second Wolfe's condition is not satisfied. 
\begin{theorem}[{\bf Wolfe's Condition}]
Let $\alpha_k$ be the step length and $p_k$ be the descent direction for minimizing the function $f(x): \mathbb{R}^n \rightarrow \mathbb{R}.$ Then the strong Wolfe's condition require $\alpha_k$ to satisfy the following two conditions: 
\begin{align}
    f(x_k + \alpha_k p_k) \leq f(x_k) + c_1 \alpha_k \nabla f_k^T p_k. \label{eqn:wolfe1} \\ 
    |\nabla f(x_k + \alpha_k p_k)^T p_k| \leq c_2 |\nabla f_k^T p_k|, \label{eqn:wolfe2}
\end{align}
with $0<c_1 <c_2 <1.$
\end{theorem}
\begin{proof}
See \cite[p. 39]{nocedal2006} for details.
\end{proof}

\begin{theorem}[{\bf First Wolfe condition}]
Under the assumptions of theorem \ref{th:1}, the update rule in Algorithm \ref{alg:algorithm} satisfies the first Wolfe's sufficient decrease condition condition \eqref{eqn:wolfe1} with $c_1 = \frac{1}{2L}.$
\end{theorem}

\begin{proof}
We recall that in Wolfe's conditions, $p_k$ and $\alpha$ are the descent direction and step size, respectively. For our method in Algorithm \ref{alg:algorithm}, $p_k = -\nabla f(x_k) $ and $\alpha = h\cot{\theta}.$ From theorem \ref{th:1} we have the expression,
\small
\begin{align*}
    f(x_{t}) &\le f(x_{t-1}) - \frac{1}{2L}(\|\nabla f(x_{t-1})\|^2)  \\
    &\le f(x_{t-1}) + \frac{1}{2L}\nabla f(x_{t-1})^Tp_{t-1}.
\end{align*}
\normalsize
Hence, for $c_1 = \frac{1}{2L}$ the above inequality holds true.
\end{proof}

\begin{table}[ht]
\begin{center}
\begin{tabular}{|l |c |c |c |c|} 
    \hline
       Optimizer & EN-B0 & EN-B0 wide & EN-B4 & EN-B4 wide  \\
       \hline
        SGD& \underline{59.53} & \underline{59.99} & \textbf{61.45} & \textbf{61.76} \\
        ADAM& 57.41 & 55.86 & 56.20 & 54.57 \\
        RMSPROP& 57.98 & 56.52 & 57.00 & 55.77 \\
        DIFFGRAD& 56.79 & 56.12 & 59.03 & 58.59 \\ 
        AGC& 58.61 & 58.33 & 58.22 & 56.46 \\
        AGT& 58.70 & 58.08 & 57.97 & 56.62 \\
        OURS& \textbf{60.21} & \textbf{60.51} & \underline{61.32} & \underline{61.17} \\
        \hline
\end{tabular}
\end{center}
\caption{\label{tab:table3_miniimagenet}Results on mini-ImageNet dataset. Since ImageNet is a relatively large dataset, we have results on EfficientNet only. We observe that our method has the best accuracy for two variants shown in {\bf bold}, and the second best results for the other two variants shown as \underline{underlined}. Despite hyperparameter tuning, most other methods except vanilla SGD perform poorly.}
\end{table}

\section{Numerical results}

\begin{table}[ht]
   \begin{center}
     
     \begin{tabular}{|l|c|c|c|c|c|c|} 
     \hline
       Method & RN18 & RN34 & RN50 & D121 & VGG16 &  DLA \\
       \hline 
       SGD       & 93.18             & 93.63             & 93.40     & 93.85             & 92.57          & 93.29   \\
       Adam      & 93.85             & 93.99             & 93.88    & 94.28             & 92.66          & 93.57   \\
       RMSProp   & 93.63             & \underline{94.07} & 93.42             & 93.84          & 92.36  & 93.29   \\
       AdaBelief  & 93.57             & 93.71             & 93.80    & 94.26             & \underline{92.84}          & 93.72 \\
       diffGrad  & 93.85             & 93.73             & 93.65    & 94.28             & 92.67          & 93.49  \\
       AGC       & 93.64             & 94.02             & 94.05    & \underline{94.57}             & 92.64          & 93.6  \\
       AGT       & \underline{93.92} & 93.89             & \underline{94.11} & 94.50          & 92.76 & \underline{93.74}  \\
       OURS      & \textbf{94.00 } & \textbf{94.24 }     & \textbf{94.39 }   & \textbf{94.75 } & \textbf{93.18} & \textbf{94.38}  \\ 
    \hline
     \end{tabular}
   \end{center}
   \caption{\label{table1_cifar10_results}Overall accuracies on CIFAR-10 first 100 epochs. The best results are in {\bf bold}, and second best is \underline{underlined}. We find that our method has the best accuracy among all the methods compared. Here RN stands for ResNet, D121 is DenseNet-121, and DLA is Deep Layer Aggregation.}
 \end{table}
 
We test the performance of the current state-of-the-art optimizers with ours on CIFAR-10 \cite{CIFAR-DATASET}, CIFAR-100 \cite{CIFAR-DATASET}, and mini-imagenet \cite{NIPS2016_90e13578}. 

For CIFAR-10, and CIFAR-100, we take the prominent image classification architectures, these are ResNet18 \cite{he2016deep}, ResNet34 \cite{he2016deep}, ResNet50 \cite{he2016deep}, VGG-16 \cite{VGG-16}, DLA \cite{8578353}, and DenseNet121 \cite{DenseNet-121}, in the tables, we refer to these architectures as RN18, RN34, RN50, VGG16, DLA, and D121 respectively. The current state-of-the-art optimizers are being compared with our method, i.e, Stochastic Gradient Descent with Momentum (SGDM), Adam \cite{ADAM}, RMSprop, Adabelief, diffGrad, cosangulargrad (AGC), tanangulargrad (AGT). 

To see long term result and stagnation of accuracy for CIFAR-100, a relatively large dataset, we run for 300 epochs and show accuracy results in Figure \ref{fig:long_run} for ResNet50, VGG-16, DenseNet-121, and DLA. Our method performs substantially better on DLA, and achieves best accuracy for resnet-50 and DenseNet on CIFAR-100. 

\begin{figure*}[ht]
  \centering
  \includegraphics[width=.32\linewidth]{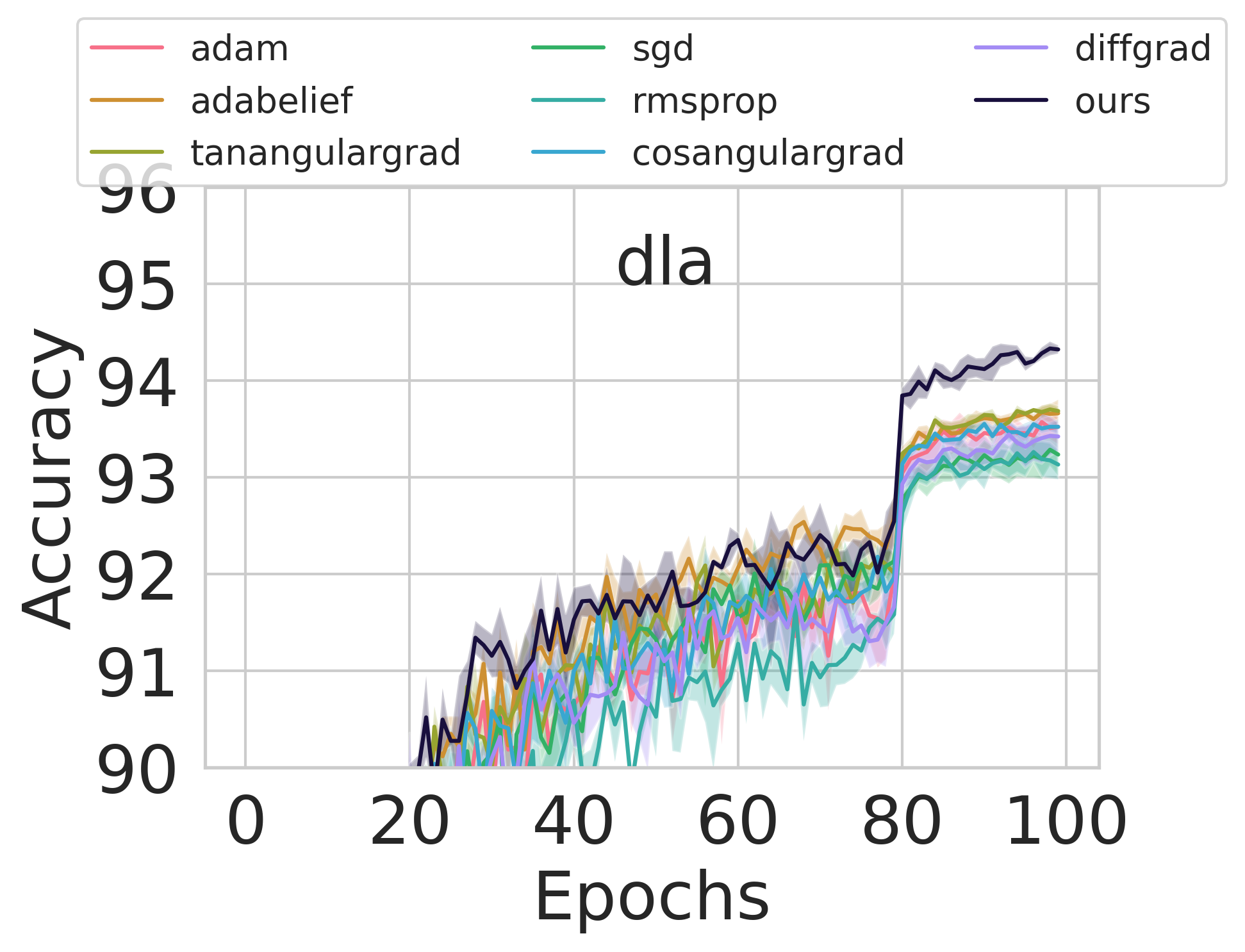}
  \includegraphics[width=.32\linewidth]{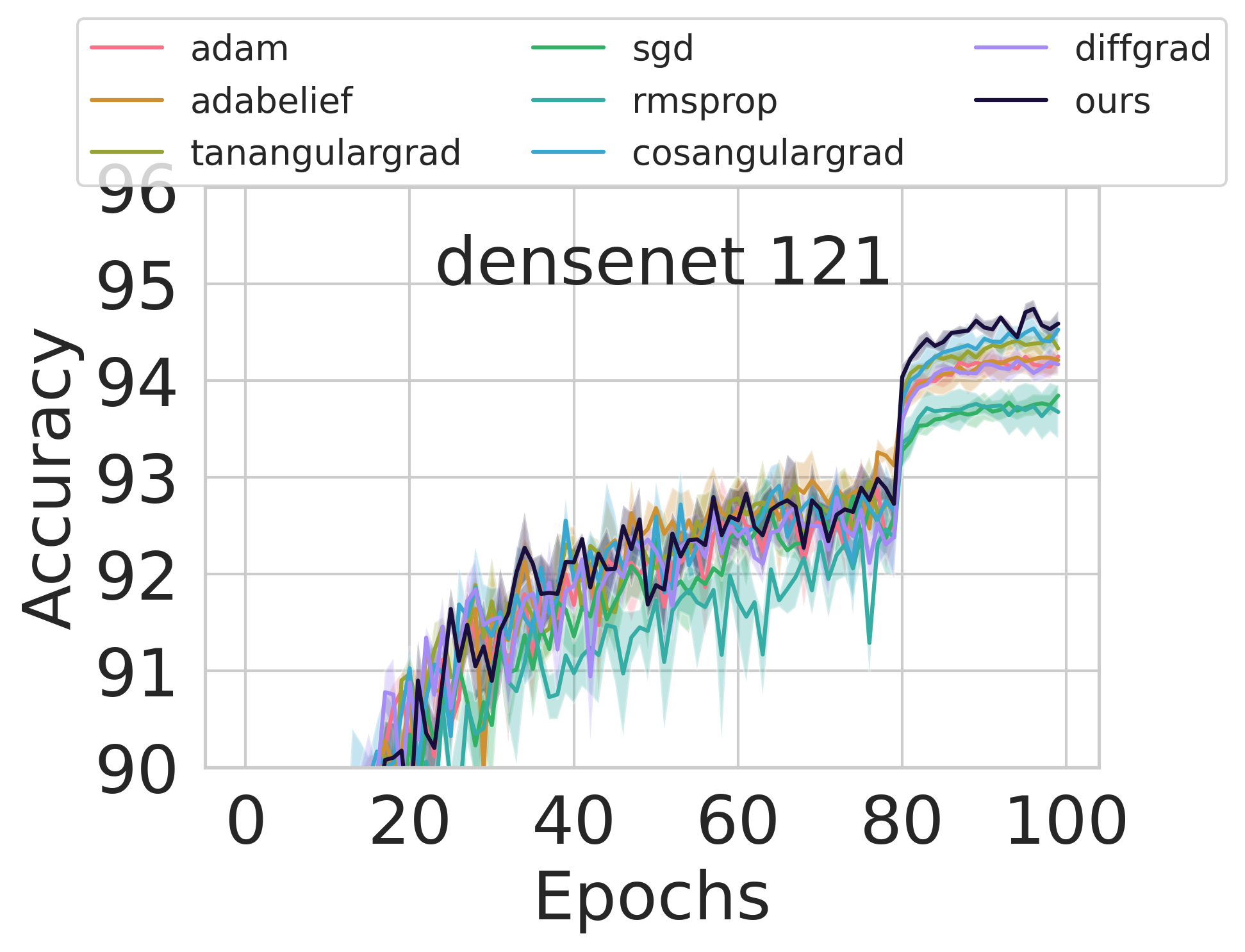}
  \includegraphics[width=.32\linewidth]{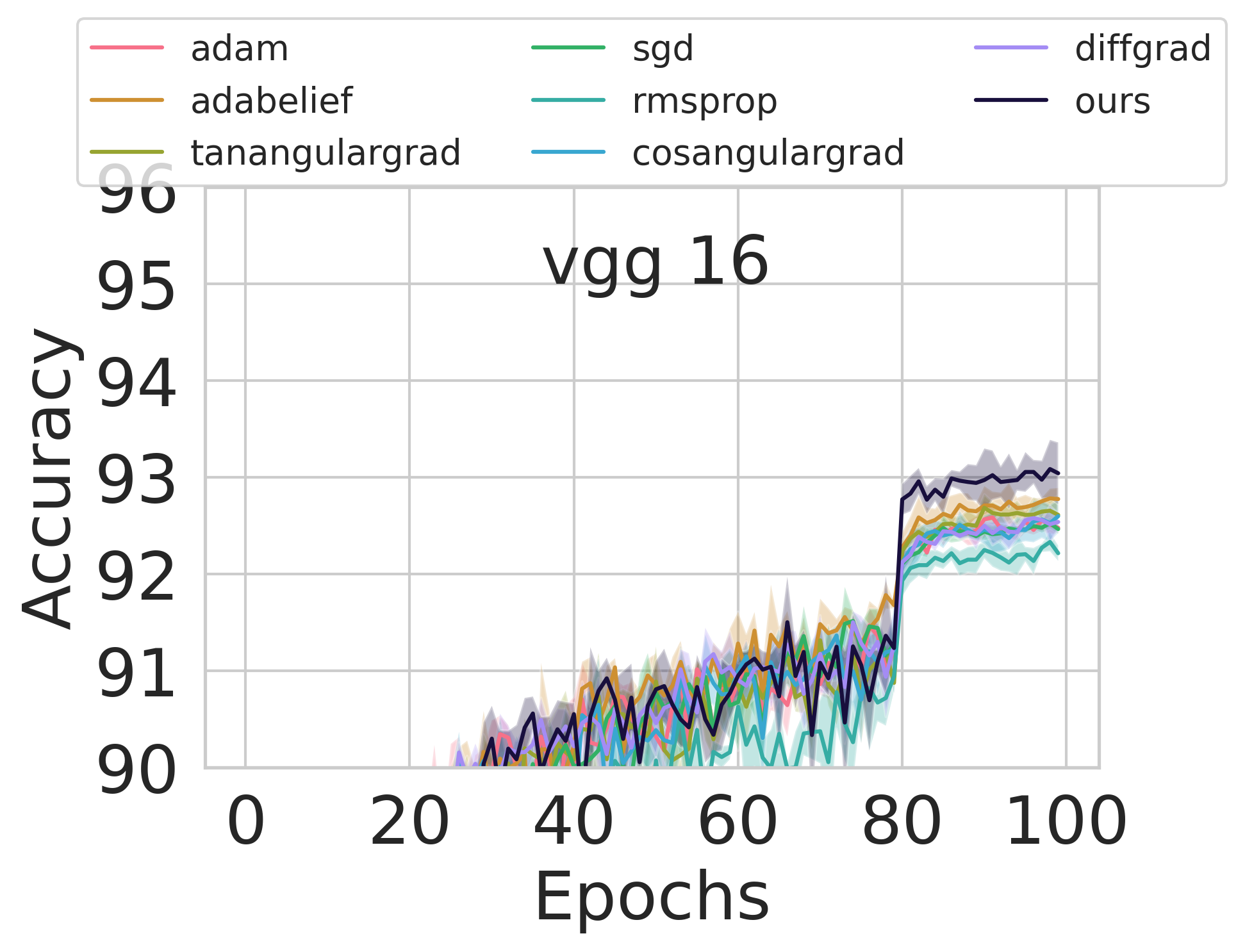}
  \includegraphics[width=.32\linewidth]{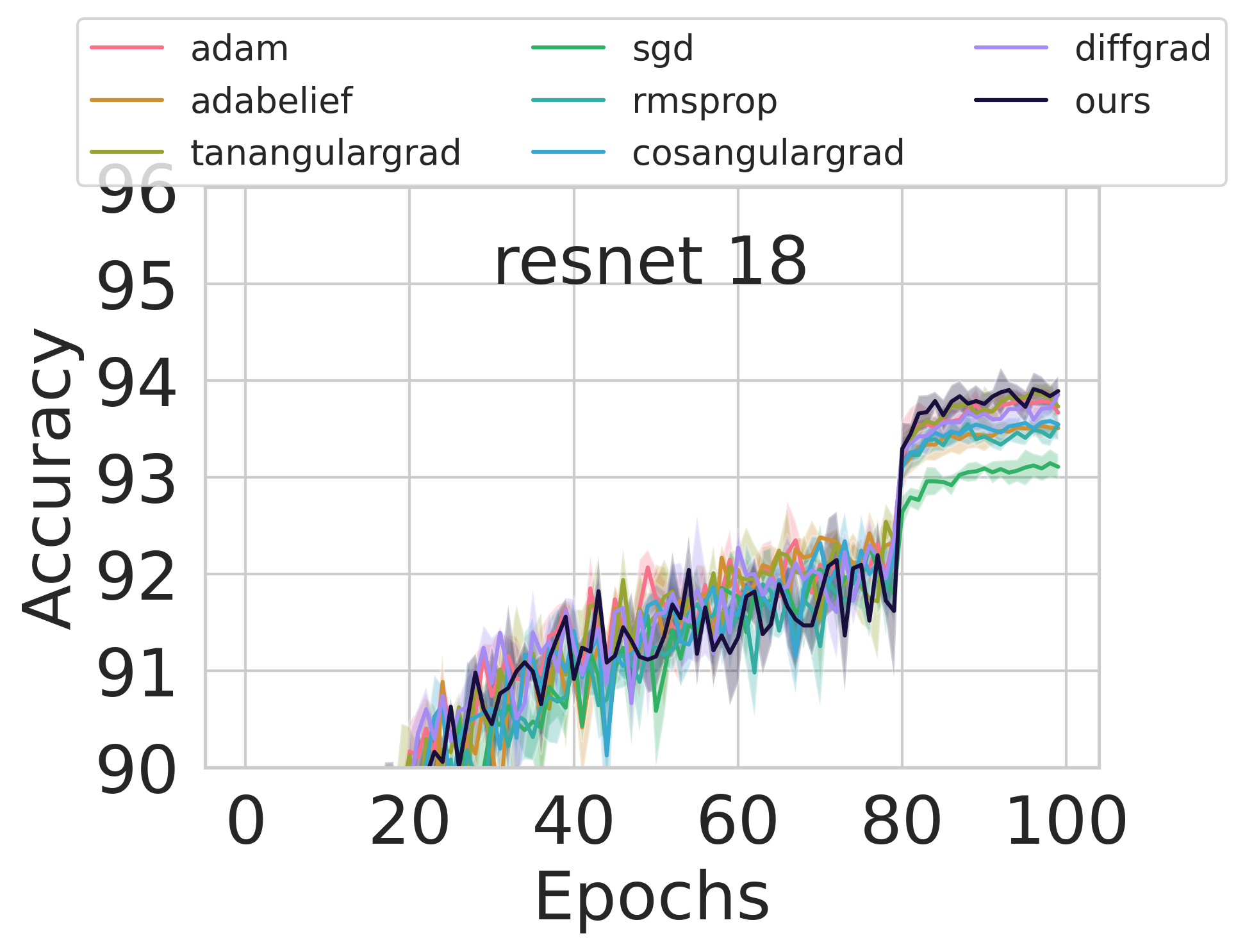}
  \includegraphics[width=.32\linewidth]{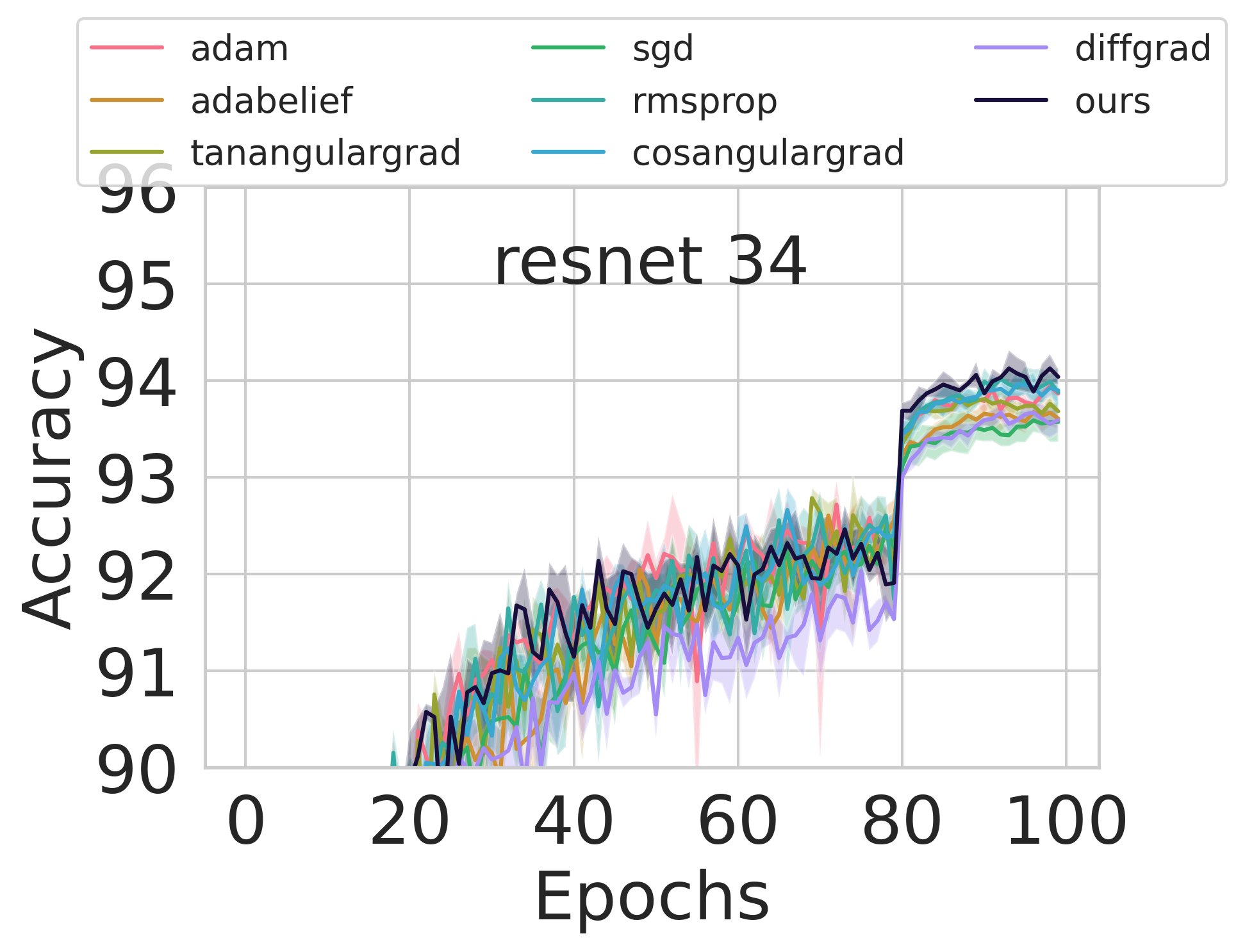}
  \includegraphics[width=.32\linewidth]{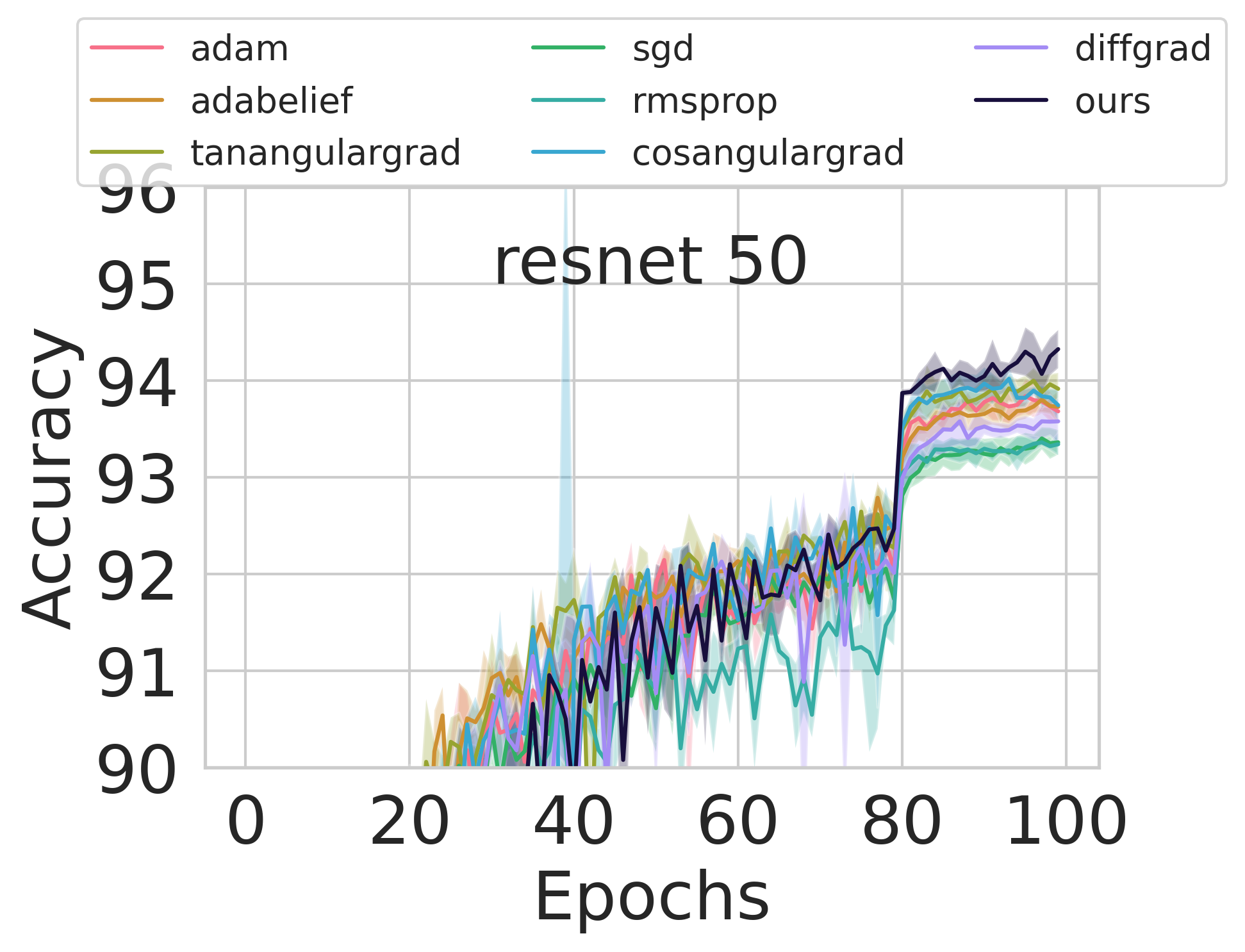}
  \caption{Accuracy plots versus epochs for image classification on the CIFAR-10 dataset. Our method is represented by dark purple curve, which is an average over 3 runs. The variance observed is not much around the mean. On DLA we observe a significantly better accuracy. On other architectures, our method maintains the highest accuracy.}
  \label{fig:cifar10 best results within 100 epochs with variance}
\end{figure*}%

Figure \ref{fig:cifar10 best results within 100 epochs with variance} shows accuracy versus epochs for CIFAR-10 for 3 runs and hence variations are seen around mean; we show the mean by bold lines. The best accuracy results for CIFAR-10 dataset are also shown in Table \ref{table1_cifar10_results}. We observe that our method performs the best across all the architectures; with more pronounced accuracy for the DLA and VGG. The observed jump in performance at epoch 80 is due to the learning rate scheduler, which is set to reduce the learning rate by 10 at the 80th epoch. The ablation study of our method on CIFAR-10 is further described in the supplementary. 

The experimental results for CIFAR-100 are shown in the Table \ref{tab:table2_cifar100}, and the plots for accuracy versus epochs for multiple runs are shown in the Figure \ref{fig:cifar100 best results within 100 epochs with variance}.  We perform the best in most of the architectures, except on ResNet-50 and our method performs the second best for VGG-16. 

Mini-imagenet \cite{vinyals2016matching} dataset is a subset of the larget Imagenet dataset \cite{5206848}. We use the efficient net architecture for mini-imagenet because as shown in Figure 1 in the paper \cite{tan2019efficientnet}, efficient-net provides the best accuracy per parameter. The best accuracy results of the mini-imagenet is shown in Table \ref{tab:table3_miniimagenet}. We outperform all the standard state-of-the-art methods that use decaying expectations of the gradient terms (includes momentum based methods) to either scale the gradient or correct its direction. We perform best in EfficientNet-B0 and EfficientNet-B0 wide; second best in EfficientNet-B4 and EfficientNet-B4 wide while performing substantially better than the adaptive methods.

\begin{table}[ht]
   \begin{center}  
     \begin{tabular}{|l|c|c|c|c|c|c|}
       \hline
       Method & RN18 & RN34 & RN50 & D121 & VGG16 &  DLA \\
       \hline
       SGD & 72.94  & 73.31 & 73.82   & 74.82  & 69.24  & 73.84 \\
       Adam & 71.98  & 71.92  & 73.52   & 74.50  & 67.78  & 71.33  \\ 
       RMSProp & 68.90  & 69.66 & 71.13   & 71.74  & 65.79  & 67.91  \\
       AdaBelief & 73.07  & 73.44  & \underline{75.47}  & \underline{75.60}  & \textbf{69.63}  & \underline{74.09} \\
       diffGrad & 72.02  & 72.21  & 74.08 & 74.53  & 67.83  & 71.07  \\
       AGC & \underline{73.25}  & \underline{73.64}  & 75.46 & 75.36  & 69.21 &  72.84 \\
       AGT & 73.13  & 73.35  & \textbf{75.75}  & 75.39 & 68.85 & 73.06  \\
       OURS & \textbf{73.33 } & \textbf{74.22 }  & 74.44    & \textbf{76.79 } & \underline{69.46}  & \textbf{74.92 } \\  
    \hline
     \end{tabular}
   \end{center}
   \caption{\label{tab:table2_cifar100}Overall accuracies on CIFAR-100 first 100 epochs. The best results are shown in {\bf bold}, and the second best results are \underline{underlined}. We observe that our method has the best results on four architectures and the second best on one architecture. Here RN stands for ResNet, D121 is DenseNet-121, and DLA is Deep Layer Aggregation.}
 \end{table}

\begin{figure*}[ht]
  \centering
  \includegraphics[width=.32\linewidth]{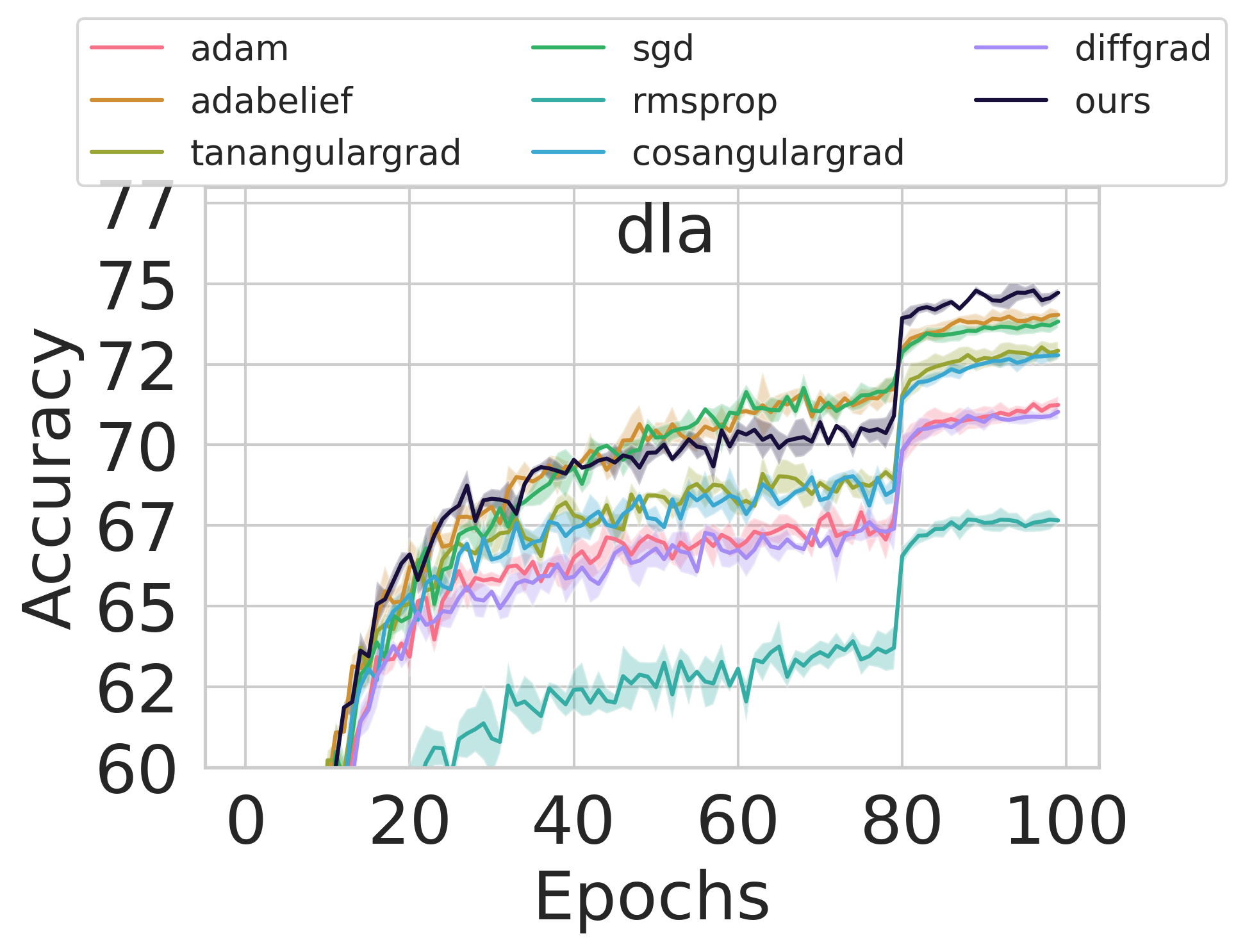}
  \includegraphics[width=.32\linewidth]{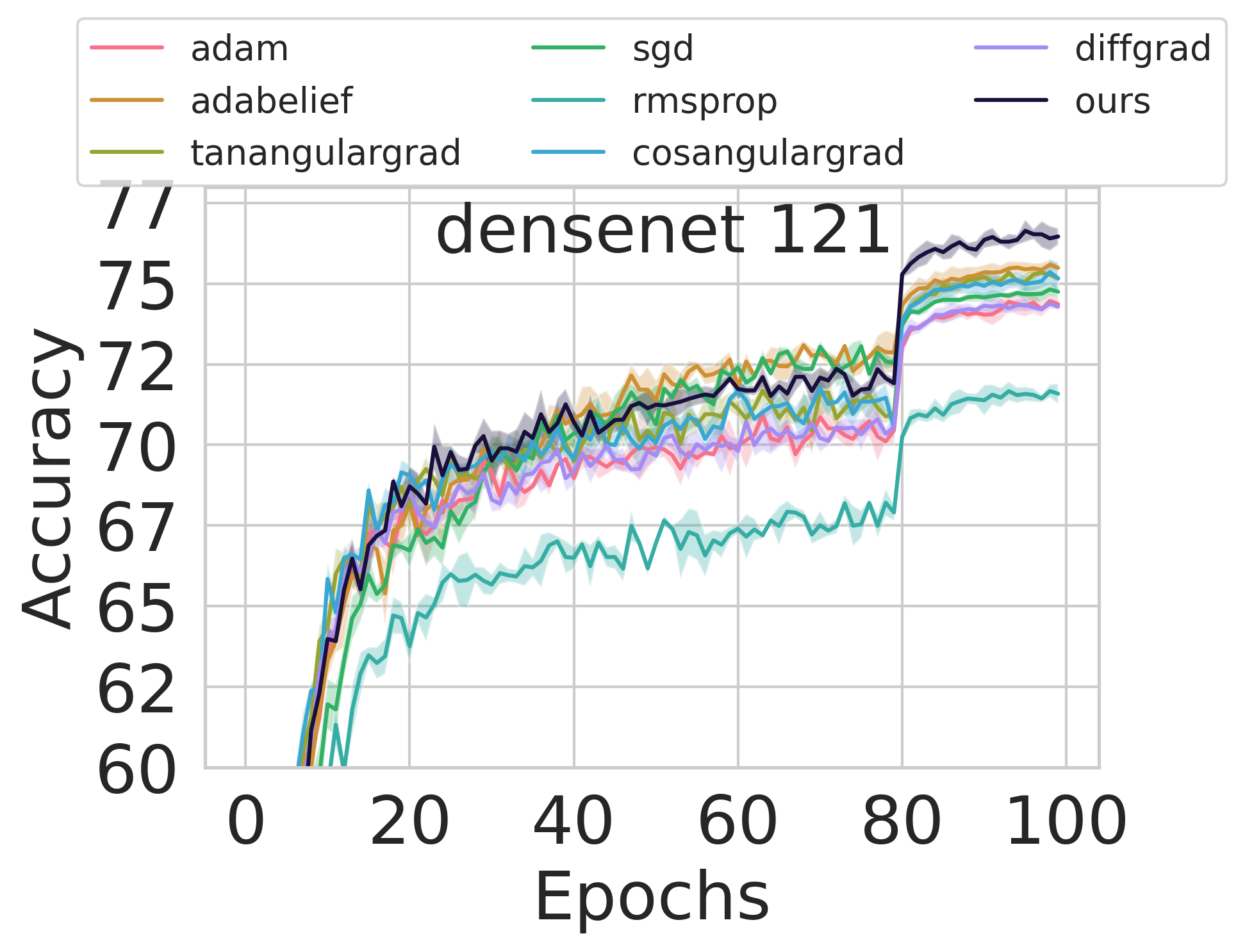}
  \includegraphics[width=.32\linewidth]{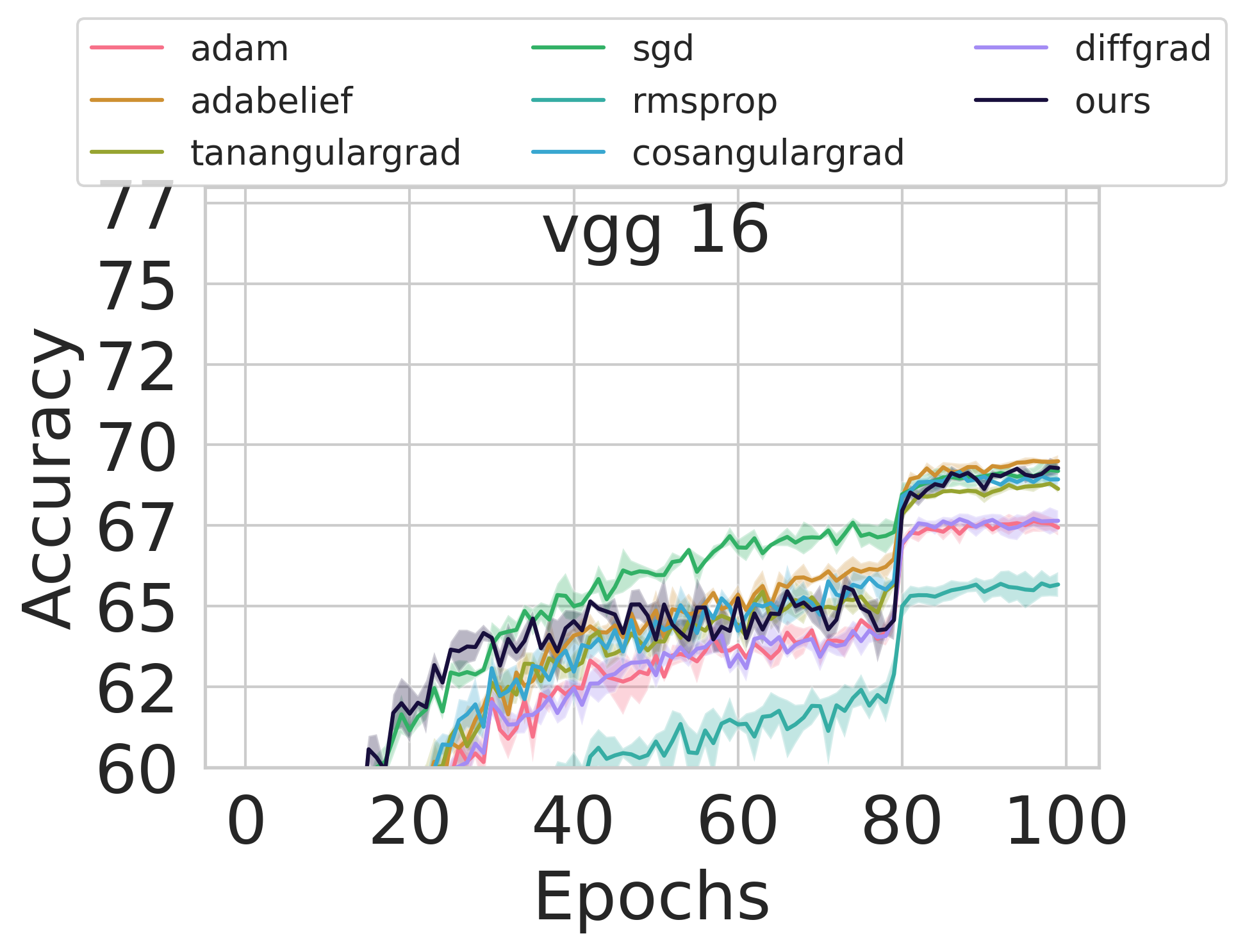}
  \includegraphics[width=.32\linewidth]{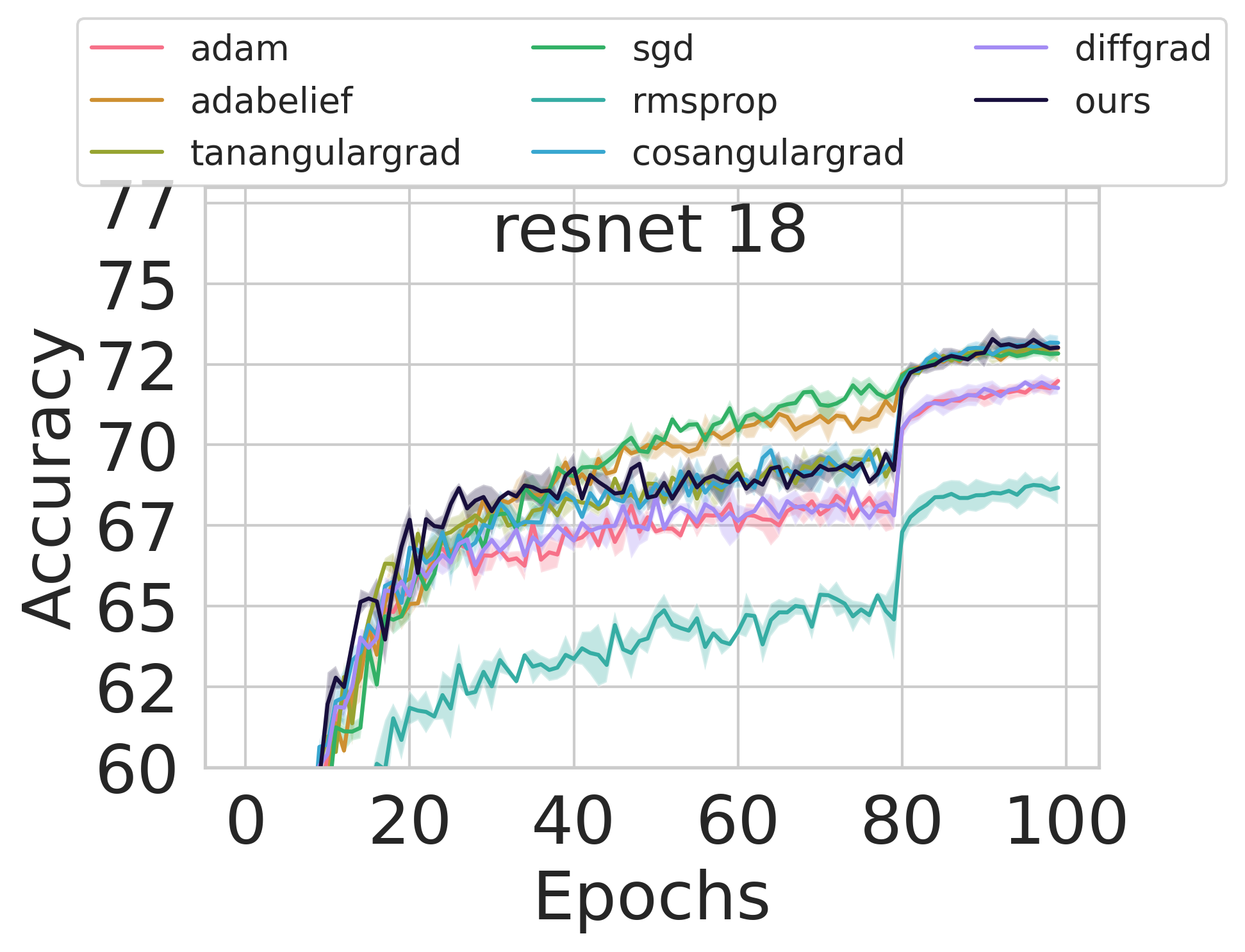}
  \includegraphics[width=.32\linewidth]{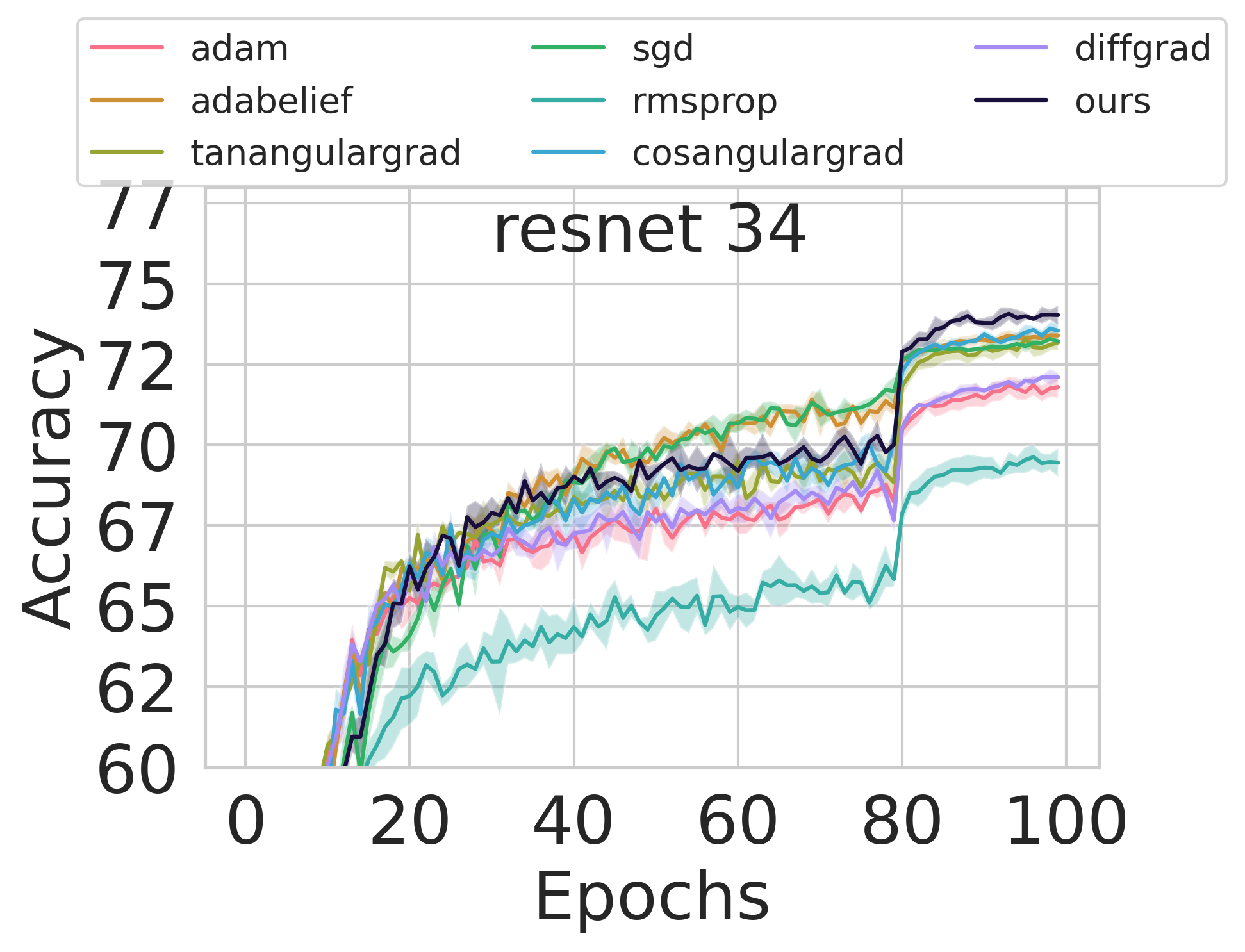}
  \includegraphics[width=.32\linewidth]{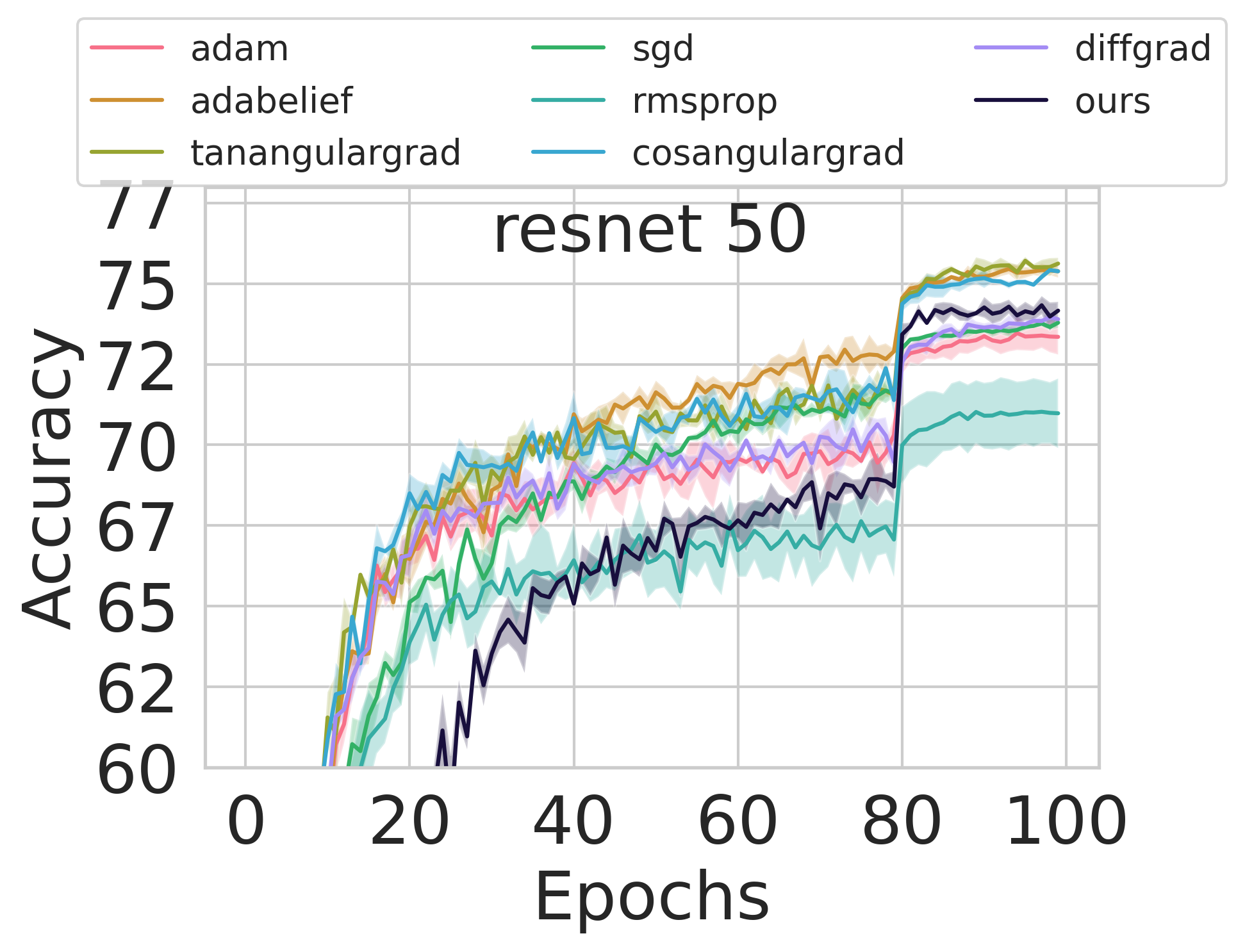}
  \caption{Accuracy plots versus epochs for image classification on CIFAR-100 dataset. Our method is represented by dark purple curve. Except for ResNet50 and VGG16, our method achieves the highest accuracy. The accuracy plots for longer number of epochs is shown in Figure \ref{fig:long_run}. For hyper-parameters and ablation study, please refer to supplementary material.}
  \label{fig:cifar100 best results within 100 epochs with variance}
\end{figure*}%

\section{Code Repository}
\href{https://github.com/misterpawan/dycent}{https://github.com/misterpawan/dycent}.

\section{Conclusion}
We propose a simple and novel method for finding an adaptive learning rate
for stochastic gradient-based descent methods for image classification tasks on prominent and recent convolution neural networks such as 
ResNet, VGG, DenseNet, EfficientNet. We observe that the proposed learning rate 
leads to higher accuracy compared to existing state-of-the-art methods such as AdaBelief, Adam, 
SGD, RMSProp, diffGrad, etc. We also show that the proposed learning rate leads to 
convergence and satisfies one of the Wolfe's condition also called Armijo's condition. The advantages of the proposed method is 
in change in learning rate which is found geometrically, and it is better than more 
costly momentum-based methods. Our future plans include analyzing and extending our method to other architectures and applications and expanding the theory to non-convex cases with a specific emphasis on improving the convergence rate.


\section*{Acknowledgement}
This work was funded by Microsoft Academic Partnership Grant (MAPG), IHUB at IIIT, and Qualcomm Faculty Award at IIIT, Hyderabad, India. We acknowledge the support of the host institute for compute resources. 


\bibliographystyle{plain}
\bibliography{aistats2023}

\end{document}